\documentclass[11pt,bezier,amstex]{article}

\usepackage[in]{fullpage}
\usepackage{microtype}
\usepackage{graphicx}
\usepackage{subfigure}
\usepackage{booktabs}

\usepackage[utf8]{inputenc}
\usepackage[english]{babel}
\usepackage{hyperref}

\usepackage{amsmath}
\usepackage{amssymb,amsopn,algorithm,algorithmic,theorem,float,bbm,bm,enumerate,color,multirow,gensymb}
\usepackage{epsfig,times,graphicx}
\usepackage{comment}
\usepackage{authblk}
\usepackage{afterpage}
\usepackage{thmtools, thm-restate}
\usepackage[square,sort,comma,numbers]{natbib}

\allowdisplaybreaks[4]

\newcommand{\beq}{\begin{equation}}
\newcommand{\eeq}{\end{equation}}

{\theoremstyle{definition} }
{\theoremstyle{definition} }

\newtheorem{theo}{Theorem}

\newtheorem{lemm}{Lemma}

\newtheorem{asmp}{Assumption}

\theoremstyle{definition}
\newtheorem{defn}{Definition}

\newcommand\I{\mathbb{I}}

\newcommand\E{\mathbb{E}}

\newcommand{\lp}{\left(}
\newcommand{\rp}{\right)}

\newcommand{\lla}{\left\langle}
\newcommand{\rra}{\right\rangle}

\newcommand{\g}{\mathbf{g}}

\renewcommand{\a}{\mathbf{a}}
\renewcommand{\b}{\mathbf{b}}

\renewcommand{\bm}{\mathbf{m}}
\renewcommand{\v}{\mathbf{v}}

\newcommand{\w}{\mathbf{w}}
\newcommand{\x}{\mathbf{x}}

\newcommand{\z}{\mathbf{z}}

\newcommand{\cM}{{\cal M}}
\newcommand{\cN}{{\cal N}}
\newcommand{\cP}{{\cal P}}

\newcommand{\cZ}{{\cal Z}}

\newcommand{\bI}{\mathbf{I}}

\newcommand{\vertiii}[1]{{\left\vert\kern-0.25ex\left\vert\kern-0.25ex\left\vert #1
    \right\vert\kern-0.25ex\right\vert\kern-0.25ex\right\vert}}

\newcommand{\proof}{\noindent{\itshape Proof:}\hspace*{1em}}

\newcommand{\qed}{\nolinebreak[1]~~~\hspace*{\fill} \rule{5pt}{5pt}\vspace*{\parskip}\vspace*{1ex}}

\newcommand {\commentout}[1] {}

\def\ints{{{\rm Z} \kern -.35em {\rm Z} }}  
\def\smallints{{{\rm Z} \kern -.3em {\rm Z} }}  
\def\pints{{{\rm I} \kern -.15em {\rm N} }}      
\newcommand{\reals}{\mathbb R}

\def\cplx{{{\rm I} \kern -.45em {\rm C} }}       
\def\l2{\rm {\mathcal L}^{2}(\reals)}            

\newcommand{\nr}{\nonumber}
\newcommand{\be}{\begin{eqnarray}}
\newcommand{\ee}{\end{eqnarray}}
\newcommand{\bea}{\begin{eqnarray}}
\newcommand{\eea}{\end{eqnarray}}
\newcommand{\beaa}{\begin{eqnarray*}}
\newcommand{\eeaa}{\end{eqnarray*}}
\newcommand{\bnad}{\begin{nad}}
\newcommand{\enad}{\end{nad}}

\renewcommand{\overline}{\bar}

\title{Private Stochastic Non-Convex Optimization:\\ Adaptive Algorithms and Tighter Generalization Bounds}

\author{
  Yingxue Zhou\thanks{equal contribution} \footnote{Department of Computer Science \& Engineering, University of Minnesota. Email: \texttt{zhou0877@umn.edu, zsw@umn.edu, banerjee@cs.umn.edu}.}, ~ Xiangyi Chen$^{*}$\footnote{Department of  Electrical \& Computer Engineering, University of Minnesota. Email: \texttt{chen5719@umn.edu, mhong@umn.edu}.},  ~ Mingyi Hong\textsuperscript{$\ddagger$}, ~ Zhiwei Steven Wu\textsuperscript{$\dagger$}, ~ Arindam Banerjee\textsuperscript{$\dagger$}
}

\date{}
\begin{document}

\maketitle

\begin{abstract}
We study differentially private (DP) algorithms for stochastic non-convex optimization. In this problem, the goal is to minimize the population loss over a $p$-dimensional space given $n$ i.i.d. samples drawn from a distribution. We improve upon the population gradient bound of  ${\sqrt{p}}/{\sqrt{n}}$ from prior work and obtain a sharper rate of $\sqrt[4]{p}/\sqrt{n}$. We obtain this rate by providing the first analyses on a collection of private gradient-based methods, including adaptive algorithms DP RMSProp and DP Adam. Our proof technique leverages the connection between differential privacy and adaptive data analysis to bound gradient estimation error at every iterate, which circumvents the worse generalization bound from the standard uniform convergence argument. Finally, we evaluate the proposed algorithms on two popular deep learning tasks and demonstrate the empirical advantages of DP adaptive gradient methods over standard DP SGD.
\end{abstract}

\section{Introduction}
\label{sec:intro}
We study differentially private algorithms for private stochastic non-convex optimization. In this problem the goal is to approximately minimize the \emph{population loss} given $n $ i.i.d. samples $\z_{1}, \dots, \z_{n}$ subject to the constraint of differential privacy~\citep{dwmc06}. Mathematically speaking, we want to privately find a model $\w^{priv}$ for solving:
\begin{equation} \label{eq: population}
 \underset{\w \in \mathbb{R}^p}{\operatorname{min}}
f(\w)  \triangleq \mathbb{E}_{\z \sim \cP} [\ell(\w, \z)]~,
\end{equation}
where $\z \in \cZ$ is a data point in the domain $\cZ$ following the unknown distribution $\mathcal{P}$, and  $\ell:\mathbb{R}^p\times \mathcal{Z}\mapsto \mathbb{R} $ is the loss function associated with the learning problem. For example, in classification problems, $\z = (\x, y)$ is an instance-label pair, $\w$ denotes the parameter of a classifier, and  $\ell(\w, \z)$ represents a surrogate loss such as cross-entropy. The goal of this problem is to find the $ \w^{priv}$ which converges to \emph{population stationarity}, i.e., small norm of \emph{population} gradient  and 
preserves differential privacy with respect to the $n$ training samples $\z_{1}, \dots, \z_{n}$ in the meanwhile.

A natural approach toward solving the problem stated in \eqref{eq: population} is Differentially Private Empirical Risk Minimization (DP-ERM) \citep{waye17, wawu19, basm14,waja19}, which finds $\w^{priv}$ by minimizing the empirical risk:
\begin{equation} \label{eq: empirical}
 \underset{\w \in \mathbb{R}^p}{\operatorname{min}}
   \hat f(\w)  \triangleq \frac{1}{n}\sum_{j =1}^{n} \ell(\w, \z_j)~,
\end{equation}
subject to differential privacy, where $\hat f(\w)$ denotes the empirical risk. For DP-ERM with non-convex loss function, the utility of the private minimizer $ \w^{priv}$ is usually measured by the $\ell_2$ norm of the empirical gradient, i.e., $\|\nabla \hat f( \w^{priv})\|$ \citep{waye17,waxu19,zhzh17}. Recent work  \citep{waye17, waxu19, waja19} solve non-convex DP-ERM by DP gradient descent and DP stochastic variance reduced gradient (SVRG) and provide $O\left(\sqrt[4]{p}/\sqrt{n}\right)$ bound on the  $\ell_2$ norm of the empirical gradient over $p$-dimensional space. Built on the empirical risk results, the standard approach for deriving bounds on the population loss is the \emph{uniform convergence} of the \emph{empirical} gradient to the \emph{population} gradient, namely an upper bound on $ \mathrm{sup}_{\w} \|\nabla f(\w) -\nabla \hat f(\w) \|$. It is known that there exist distributions over over $p$-dimensional space for which the best result on uniform convergence is $O\left(\sqrt{p}/\sqrt{n}\right)$ \citep{meba16,fose18}. \cite{waxu19} leverages this result and give the state-of-the-art upper bound $O\left(\sqrt{p}/\sqrt{n}\right)$ on the $\ell_2$-norm of the population gradient.

In this work, we generalize the DP gradient descent algorithms \citep{waye17,waxu19} for non-convex optimization along with
popular gradient-based algorithms, 
including DP GD, DP RMSprop, and DP Adam. We provide population risk analysis for all these algorithms. Specifically, we obtain a bound of $\tilde O\left(\sqrt[4]{p}/\sqrt{n}\right)$ on the $\ell_2$-norm of the population gradient, showing that the known bound $O(\sqrt{p}/\sqrt{n})$ given by  \cite{waxu19} is suboptimal. We get the sharper bound by leveraging the advantage of generalization properties of differential privacy itself.  In particular, our approach to the population risk analysis, i.e., bound on the population gradient, relies on the generalization properties of differential privacy and adaptive data analysis (ADA) \citep{dwfe2015a,dwfe2015b,dwfe2015c} to bound the gap between the empirical gradient and population gradient at every iterate.  Mathematically, we show that differentially private gradients approximate the population gradients with high probability \emph{across all
iterations}, leading to high probability bounds on the $\ell_2$ norm of the population gradient, i.e., $\|\nabla f(\w) \|$.

We further provide a lower bound of gradient uniform convergence rate that matches rate of $\sqrt{p}/\sqrt{n}$ in \cite{waxu19}. This indicates that in order to improve the current population gradient bound in \cite{waxu19}, it is necessary to sidestep the uniform convergence argument in prior work.

We also provide an empirical risk analysis that bounds the empirical gradient norm for DPAGD algorithms, including DP RMSprop and DP Adam. To our best knowlege, we are the first to provide the first empirical risk analyses for DP variants of these adaptive gradient methods. Finally, we empirically evaluate DP SGD, DP Adam, and DP RMSprop on two popular deep learning tasks. Our experiments demonstrate that the adaptive methods of DP Adam and DP RMSprop tend to outperform the stanard DP SGD method.

The remainder of this paper is organized as follows.
Section \ref{sec:related} and Section \ref{sec:pre} describe related work and preliminaries, respectively. The DP adaptive algorithms and corresponding population risk analyses are described in Section \ref{sec:algorithm}. Section \ref{sec:erm} discusses the empirical risk analysis and the uniform convergence.
Section \ref{sec:experiments} shows our experimental results.
Section \ref{sec:conclusion} concludes our work. 
All the proofs are deferred to the Appendix.

\section{Related work}
\label{sec:related}

{\bf DP-ERM and Generalization:}  
DP-ERM has been well-studied in the last decade. Algorithms such as output-perturbation that perturbs the output of a non-DP algorithm, objective function perturbation that perturbs the objective function \citep{chmo09} and gradient perturbation that adds noise to the gradient in gradient descent algorithms \citep{SongCS13, basm14} have been proposed to solve DP-ERM. We mainly discuss those algorithms that are most related to our problem, i.e., gradient perturbation \citep{basm14,bafe19, waxu19,waye17,zhzh17, waja19, budl19}. Most DP gradient-based algorithms focus on minimizing the convex loss and aim to achieve optimal {empirical} and population risk  bounds under privacy.  \cite{basm14} propose DP gradient descent algorithms and apply \emph{uniform convergence} \citep{shsh09} of empirical loss to population loss, i.e., $\sup _{\w}(f(\w) -\hat f(\w))$ to obtain a generalization bound on the population risk. Afterward, \cite{bafe19}
derive an optimal bound on the population risk using the generalization properties of \emph{uniform stability} \citep{boel02} of a standard noisy mini-batch stochastic gradient descent. More recently, \cite{FeldmanKT20} further reduces the number of gradient computations in the algorithm of \cite{bafe19}.

Recently, DP algorithms have been studied for non-convex loss functions \citep{waye17,waxu19,zhzh17, waja19}. Since finding the global minimum for non-convex functions is NP-hard, the utility of a DP algorithm is typically measured by the $\ell_2$-norm of the gradient. \cite{waye17,waxu19,zhzh17,waja19} show that a bound of $O(\sqrt[4]{p}/\sqrt{n})$ on the $\ell_2$-norm of the \emph{empirical gradient} can be achieved by DP gradient descent and DP SVRG algorithms. 
\cite{waxu19} extend the bound from empirical gradient to population gradient by using uniform convergence \citep{meba16}, i.e., $ \mathrm{sup}_{\w} \|\nabla f(\w) -\nabla \hat f(\w) \|$ which leads to a suboptimal rate of $O(\sqrt{p}/\sqrt{n})$.

{\bf Adaptive Data Analysis:} In adaptive data analysis (ADA),  an analyst \emph{reuses} a dataset to generate hypotheses (e.g., statistical queries) and validate the results on the same dataset. The pioneering work of \citep{dwfe2015a, dwfe2015b, dwfe2015c} provides a transfer theorem showing that as long as the value of a hypothesis selected by a DP algorithm is close to the true empirical value, its value evaluated on the dataset is close to its true value in the population.
Later \cite{bani16, JungLN0SS20} further improve and simplify the analysis for the transfer theorem. In our setting, the gradients across the iterations can be viewed as a sequence of adaptively chosen queries, and so we can bound the estimation error of these gradient queries as well. \cite{zhch2018} leverages similar techniques for convex optimization.

{\bf Adaptive Gradient Methods:}
Adaptive gradient methods usually refer to a class of algorithms that change learning rates adaptively during optimization. Representative methods in this class include AdaGrad \citep{duha11}, RMSProp \citep{tihi12},  Adam \citep{kiba14}, and AMSGrad \citep{sare18}, which use the second moment of gradients to change the learning rates on different coordinates to adapt to the geometry of the loss function. In the non-convex setting, existing work provide $O(1/\sqrt{T})$  convergence bound of the objective gradient \citep{ghla13, zare18, wawu19}
with $T$ stochastic gradient computations. The DP variants of these algorithms are increasingly popular, but no convergence guarantee has been established. Our work provides the first known convergence proofs for these algorithms.

\section{Preliminaries}
\label{sec:pre}
\textbf{Notations:} We use $\g_t$ and $\nabla  f(\w_t)$ interchangeably to denote population gradient, i.e., $\g_t = \nabla  f(\w_t) = \mathbb{E}_{\z \in \cP} [\nabla \ell(\w_t, \z)]$. We also use 
$\nabla \hat f(\w) = \frac{1}{n}\sum_{j=1}^n \nabla \ell(\w_t, \z_j)$ and $\hat \g_t$ interchangeably denotes the empirical gradient evaluated on $n$ training samples $S$, i.e., $S = \left\{\z_{1}, \dots, \z_{n}\right\}$.
For a vector $\v \in \mathbb{R}^p$, $\v^2$ denotes element-wise product. Either $\v^i$ or $[\v]_i$ are used to denote the $i$-th coordinate of $\v$, where $i \in [p]$. 
$\|\v\|$ denotes the $\ell_2$-norm of $\v$.  For a scalar $a$ and vector $\v$, $\v + a$, $\v/a$ denotes element-wise addition and division, and $\min(\v, a)$ means element-wise operation such that $\min(\v^i, a)$ for every coordinate $i \in [p]$.

\begin{defn} \label{def:dp} 
(Differential Privacy \citep{dwmc06})
A randomized algorithm $\mathcal{M}$ is $(\epsilon, \delta)$-\emph{differentially private} \citep{dwmc06} if  for any pair of datasets $S, S'$ differ in exactly one data point and for all event  $\mathcal{Y}\subseteq Range(\mathcal{M})$ in the output range of  $\mathcal{M}$, we have 
\begin{equation}
P\{\mathcal{M}(S)\in \mathcal{Y}\} \leq \exp(\epsilon)P\{\mathcal{M}(S')\in \mathcal{Y} \} + \delta,   
\end{equation}
where the probability is taken over the randomness of $\cM$.
\end{defn}

Intuitively, the definition of differential privacy means that the outcomes of two nearly identical datasets (different on a single component) should be nearly identical such that an analyst will not be able to distinguish any single data point by monitoring the change of the output. 
Differential privacy has several properties that make it particularly useful in applications such as Advanced Composition \citep{dwro2014} and and Moments Accountant (MA) \citep{abch16} which give the privacy analysis of adaptive composition of private mechanisms.

We make the following assumptions about the objective function throughout the paper. 
\begin{asmp} \label{asmp: bounded_gradient}
        The individual gradient is bounded, i.e.,  for any $\z \in \cZ$ and any $\w \in \mathbb{R}^p$,
    \begin{equation}
\| \nabla \ell (\w, \z) \| \leq G.
    \end{equation}
\end{asmp}

Note that this assumption implies that the population gradient and empirical gradient are also bounded as $\| \nabla f(\w)\| \leq G$ and $\| \nabla \hat f(\w)\| \leq G$.

\begin{asmp} \label{asmp: smoothness} Loss function 
    $ \ell: \mathbb{R}^p \rightarrow \mathbb{R}$ is differentiable (but not necessarily convex), bounded from below by $\ell^\star$,
    and has L-Lipschitz gradient, i.e.,
    \begin{equation}
    \| \nabla \ell(\w) -\nabla \ell(\w^\prime) \| \leq L \|\w-\w^\prime \|, ~ \forall \w, \w^\prime \in \mathbb{R}^p.
    \end{equation}
\end{asmp}
Assumption \ref{asmp: smoothness} implies the population loss $f(\w)$ and empirical loss $\hat f(\w)$ also have  L-Lipschitz gradient and bounded from below.

\section{Private Adaptive Gradient Descent}
\label{sec:algorithm}
In this section, we first present a general framework of DP adaptive gradient descent algorithms (DPAGD) that capture DP GD, DP RMSprop, and DP Adam as special cases. Then we present the privacy guarantee of DPAGD. Later, we discuss the generalization guarantee achieved by differential privacy. Finally, we give the theoretical analysis, i.e., the bound on the $\ell_2$-norm of the population gradient 
$\| \nabla f(\w) \|$.

\begin{algorithm}[h] 
\caption{DPAGD: Differentially Private Adaptive Gradient Descent}
\begin{algorithmic}[1] \label{algo: full batch}
\STATE \textbf{Input}: Dataset $S$,  loss $\ell(\cdot)$, initial point $\w_0$, sequence of functions $\{\phi_t, \psi_t \}_{t=1}^T$.
  \STATE \textbf{Set}:  Noise parameter $\sigma$, iteration time $T$,  step size parameters $\eta_t$, $\nu$, $\lambda$.
	\FOR{$t = 0,...,T$}
	\STATE  Compute noisy gradient $\tilde \g_t = \mathbb{E}_{\z \in S}\nabla \ell(\w_t, \z) + \b_t$,  where $\b_t\sim \cN(0, \sigma^2\bI_p)$.
	\STATE ${\bf m}_t = \phi_t(\tilde \g_1, ..., \tilde \g_t)$ and $\v_t = \min(\psi_t(\tilde \g_1,.., \tilde \g_t), \lambda)$
	\STATE $\w_{t+1}=\w_{t}-\eta_t {\bf m}_t /(\sqrt{\v_t}+\nu)$.
		\ENDFOR 
	\end{algorithmic}
\end{algorithm}

\begin{table}[h]
    \caption{An overview of DP adaptive gradient algorithms.}
    \centering
    \footnotesize{
    \begin{tabular}{cccc}
        \hline
        & DP GD  & DP RMSprop & DP  Adam \\ \hline
    $\phi_{t}$   &  $\tilde \g_{t}$ & $\tilde \g_{t}$  &  $\left(1-\beta_{1}\right) \sum_{j=1}^{t} \beta_{1}^{t-j} \tilde \g_{j}$ \\ 
    ${\psi_{t}}$    &  $\mathbb{I}$  & $ \left(1-\beta_{2}\right) \sum_{j=1}^{t} \beta_{2}^{t-j} \tilde \g_{j}^{2}$ & $ \left(1-\beta_{2}\right) \sum_{j=1}^{t} \beta_{2}^{t-j} \tilde \g_{j}^{2}$ \\ \hline
    \end{tabular}}
    \label{table:algo}
\end{table}

We use Algorithm \ref{algo: full batch} to provide a generic adaptive framework of the DPAGD. Given $n$ training samples $S$, loss function $\ell$, at each iteration $t \in [T]$, Algorithm \ref{algo: full batch} first computes gradient $\hat \g_t = \mathbb{E}_{\z \in S}\nabla \ell(\w_t, \z)$. Then Algorithm \ref{algo: full batch} adds i.i.d.~Gaussian noise to the gradient $\tilde \g_t = \mathbb{E}_{\z \in S}\nabla \ell(\w_t, \z) + \b_t$,  where $\b_t\sim \cN(0, \sigma^2\bI_p)$ (line 4). 
Afterward, Algorithm \ref{algo: full batch} updates the $\w_{t+1}$ based on $\phi_t$ and $\psi_t$ that are functions of past noisy gradients $\tilde \g_1, ..., \tilde \g_t$ (line 5, 6). We specify the  ``averaging'' functions $\phi_t$ and $\psi_t$ for different adaptive gradient algorithms, i.e., DP GD, DP RMSprop and DP Adam in Table \ref{table:algo}. 

The difference between non-private adaptive gradient descent and DP adaptive gradient descent is that DPAGD uses the noisy gradient $\tilde \g_t$ instead of sample gradient $\hat \g_t$ in the ``averaging'' functions $\phi_t$ and $\psi_t$ to update the step size and parameter $\w_{t+1}$. Note that $\psi_t$ in Table \ref{table:algo} for DP RMSProp and DP Adam is the exponentially decaying average of the square of the past noisy gradients, which can be extremely large due to the injected noise, leading to a vanished step size $\eta_t/(\sqrt{\v_t}+ \nu)$. Thus, Algorithm \ref{algo: full batch} clips the $\psi_t$ by a threshold $\lambda$ coordinate-wisely, where $\lambda >0$ is a hyper-parameter.

In this work, we mainly focus on DP GD, DP RMSprop, and DP Adam (see details in Table \ref{table:algo}).
Note that noise variance $\sigma^2$, step size $\eta_t >0$, and iteration number $T$, $ 0 < \beta_1,~ \beta_2 < 1,~\nu \geq 0, \lambda > 0$ are the parameters of Algorithm \ref{algo: full batch}. We present the optimal values of them for DP GD, DP RMSprop, and DP Adam, respectively in the subsequent sections.

\begin{restatable}{theo}{theodpfb}
\label{thm:dp_analysis_fb}
    (Privacy guarantee) There exist constants $c_1$ and $c_2$ so that given the number of iterations $T$, for any $\epsilon \leq c_1  T$, DPAGD (Algorithm~\ref{algo: full batch}) is $(\epsilon, \delta)$-differentially private for any $\delta >0$ if
    \begin{eqnarray} 
    \label{eq: sigma_fb}
    \sigma^{2} \geq c_2 \frac{ G^{2} T \ln \left(\frac{1}{\delta}\right)}{n^{2} \epsilon^{2}}.
    \end{eqnarray}
\end{restatable}

Theorem \ref{thm:dp_analysis_fb} is a variant of Theorem 1 in \citep{abch16} where 
the variance of noise is derived by moments accountant (MA) \citep{abch16,waba19}.
MA is a method to calculate the privacy cost for a composition of differential private mechanisms which has sharper bound on $\epsilon$ and $\delta$. DPAGD is a composition of $T$ Gaussian Mechanism (line 4 in Algorithm \ref{algo: full batch}).
MA allows DPAGD to save a factor of $\ln (T / \delta)$
on the variance of noise compared with
those achieved by using the Advanced Composition \citep{dwro2014}.

\subsection{Generalization guarantee of differential privacy}
\label{sec:gen_of_dp}

To analyze the convergence of DPAGD in terms of the $\ell_2$ norm of the population gradient, we need to bound the gradient estimation error between population gradient $\g_t$ and noisy gradient $\tilde \g_t$, i.e., $\| \tilde \g_t - \g_t\|$. To bound this error, one needs to bound the \emph{generalization error} between population gradient $\g_t$ and empirical gradient $\hat \g_t$ as well as the noise $\b_t$, i.e., $\|\tilde \g_t -\g_t \| \leq \| \hat \g_t - \g_t\| + \|\b_t \|$ at every iteration $t$. Usually the deviation bound of $\| \hat \g_t - \g_t\|$ can be estimated by the Hoeffding’s bound, i.e.,  for an initial $\w_0$ which is independent of the dataset $S$, we have $ \mathbb{P}\{|\hat \g^i_0 - \g_0^i | \geq \mu \} \leq 2 \exp \left(\frac{-2n\mu^2}{4G_\infty^2} \right)$,  $\forall i \in [p]$ and $\mu > 0$,  where $G_\infty$ is the  $\ell_\infty$-norm of the gradient $ \g_0$. However, in general, this concentration bound will not hold for $\w_t, \forall t> 0$  since $\w_t$ is no longer independent of dataset $S$. Since DPAGD is differentially private, we use the generalization property of differential privacy itself to provide the gradient concentration bound which holds even though the $\w_t, \forall t> 0$ are adaptively generated on the same dataset $S$ (Theorem \ref{thm: acc_basic_fb}).

\begin{restatable}{theo}{theoaccbasicfb}
\label{thm: acc_basic_fb} 
In DPAGD, set $\sigma$ to be as  \eqref{eq: sigma_fb}, and for any $\mu> 0$, $\epsilon$, $\delta$
and sample size $n$ satisfying $\epsilon \leq \frac{\sigma}{13}$, $\delta \leq \frac{\sigma \exp(-\mu^2/2)}{13 \ln(26/\sigma)}$ and $n \geq \frac{2\ln(8/\delta)}{\epsilon^2}$, the noisy gradients $\tilde \g_1,...,  \tilde \g_T$ produced in Algorithm \ref{algo: full batch} satisfy
\begin{equation}
    \mathbb{P}\left\{\|\tilde \g_t - \g_t\| \geq \sqrt{p}\sigma(1+\mu)\right\}  \leq 4p\exp(-\mu^2/2) 
\end{equation}
for all $t \in [T].$
\end{restatable}

Theorem \ref{thm: acc_basic_fb} indicates that gradient $\tilde \g_t$
produced by DPAGD is concentrated around population gradient $\g_t$ with a tight concentration error bound $\sqrt{p}\sigma(1+\mu)$. The noise variance $\sigma$ illustrates a trade-off between privacy and accuracy: 
A higher noise level $\sigma$ brings a better privacy 
guarantee (i.e., a smaller $\epsilon$), but meanwhile incurs a 
larger concentration error $\sqrt{p}\sigma(1+\mu)$. 

To obtain a generalization bound (i.e., the upper bound on the $\ell_2$-norm of the population gradient) of Algorithm \ref{algo: full batch} with the guarantee of being $(\epsilon,\delta)$-differential private, we set the parameter $\sigma$, iteration $T$ in Algorithm \ref{algo: full batch} to satisfy the conditions in Theorem \ref{thm: acc_basic_fb}, which also brings out an requirement on the sample size $n$. We present the details and the convergence of the population gradient for DP GD, DP RMSprop and DP Adam in the following section.

\subsection{Convergence of the population gradient}
\label{sec:con_rate}

In this section, we present the convergence rate of Algorithm \ref{algo: full batch}. We consider different choice of $\phi_t$ and $\psi_i$ as stated in Table \ref{table:algo}. Note that for $\phi_t(\tilde \g_1,...,\tilde \g_t) = \tilde \g_t$,  $\psi_t(\tilde \g_1,...,\tilde \g_t) = \I$, Algorithm \ref{algo: full batch} represents DP GD, which recovers the Algorithm 4 in \cite{waye17}. For $\phi_t(\tilde \g_1,...,\tilde \g_t) = \tilde \g_t$, and $ \psi_t = \left(1-\beta_{2}\right) \sum_{j=1}^{t} \beta_{2}^{t-j} \tilde \g_{j}^{2}$, Algorithm \ref{algo: full batch} represents DP RMSProp.  For $\phi_t(\tilde \g_1,...,\tilde \g_t) = \left(1-\beta_{1}\right) \sum_{j=1}^{t} \beta_{1}^{t-j} \tilde \g_{j}$, and $ \psi_t = \left(1-\beta_{2}\right) \sum_{j=1}^{t} \beta_{2}^{t-j} \tilde \g_{j}^{2}$, Algorithm \ref{algo: full batch} represents DP Adam, which is similar to the noisy adam algorithm in \cite{budl19}. In the following theorem, we present the convergence rate of DP GD, DP RMSprop and DP Adam respectively.

\begin{restatable}{theo}{theogenboundfb}
\label{thm:gen_bound_fb} 
(Population risk analysis) Under the Assumption \ref{asmp: bounded_gradient} and \ref{asmp: smoothness}, given training sample $S$ of size $n$, for any $\epsilon, \delta >0$ and $n \geq \frac{2\ln(8/\delta)}{\epsilon^2}$, 
 set $\sigma$ in  Alorithm \ref{algo: full batch} be as \eqref{eq: sigma_fb}, for any $\beta >0$, 
\begin{enumerate}
    \item ({\bf DP GD}) Algorithm \ref{algo: full batch} with $\phi_t(\tilde \g_1,...,\tilde \g_t) = \tilde \g_t$,  $\psi_t(\tilde \g_1,...,\tilde \g_t) = \I$, $\nu = 0$, $\lambda = 1$, $T = \frac{n\epsilon \sqrt{L}}{G \sqrt{p \ln(1/\delta)}}$, and step size $\eta_t = \frac{1}{4L}$
satisfies, 
\vspace{-2mm}
\begin{equation} 
 \mathbb{E} \| \nabla f(\w_R)\|^2  \leq    O\left( \frac{G\sqrt{pL\ln (1/\delta)}\ln(np\epsilon/\beta)}{n\epsilon}\right) 
\end{equation}
with probability at least $1-\beta$, where $\w_R$ is uniformly sampled from $\{\w_1, \w_2, ...,\w_T\}$ and the expectation is over the draw of $\w_R$;

     \item  ({\bf DP RMSprop}) Algorithm \ref{algo: full batch}  with $\phi_t(\tilde \g_1,...,\tilde \g_t) = \tilde \g_t$, and $ \psi_t = \left(1-\beta_{2}\right) \sum_{j=1}^{t} \beta_{2}^{t-j} \tilde \g_{j}^{2}$, $T = \frac{n\epsilon }{G \sqrt{p \ln(1/\delta)}}$, step size $\eta_t = \eta$,  $0 < \beta_2 < 1$, $\lambda > 0$, parameters $\nu$ and $\eta$ are chosen such that: $\eta \leq \frac{\nu}{4L}$
satisfies, \vspace{-2mm}
\begin{equation} 
\mathbb{E}\|\nabla f(\w_R)\|^2 \leq
   O\left( \frac{G \sqrt{p\ln (1/\delta)}\ln(np\epsilon/\beta)}{n\epsilon}\right) 
\end{equation}
with probability at least $1-\beta$, where $\w_R$ is uniformly sampled from $\{\w_1, \w_2, ...,\w_T\}$ and the expectation is over the draw of $\w_R$;

      \item ({\bf DP Adam}) Algorithm \ref{algo: full batch}  with $\phi_t(\tilde \g_1,...,\tilde \g_t) = \left(1-\beta_{1}\right) \sum_{j=1}^{t} \beta_{1}^{t-j} \tilde \g_{j}$, and $ \psi_t = \left(1-\beta_{2}\right) \sum_{j=1}^{t} \beta_{2}^{t-j} \tilde \g_{j}^{2}$, $T = \frac{n\epsilon }{G \sqrt{p \ln(1/\delta)}}$, step size $\eta_t = \eta$, $0<\beta_2<1$, $\lambda > 0$, $\beta_1$ and $\nu$ are chosen such that: 
    $\eta \leq ({\sqrt{1/2+4\beta_1/(1-\beta_1)^2}-1/2})\frac{(1-\beta_1)^2}{\beta_1}\frac{\nu}{4L}$ 
satisfies, 
\begin{equation} 
\mathbb{E}\|\nabla f(\w_R)\|^2 \leq
   O\left( \frac{G\sqrt{p\ln (1/\delta)}\ln(np\epsilon/\beta)}{n\epsilon}\right) 
\end{equation}
with probability at least $1-\beta$, where $\w_R$ is uniformly sampled from $\{\w_1, \w_2, ...,\w_T\}$ and the expectation is over the draw of $\w_R$.
\end{enumerate}
\end{restatable}

Theorem \ref{thm:gen_bound_fb} shows that DP RMSprop and DP Adam as well as DP GD achieve bound $\tilde O(\frac{ \sqrt{p}}{n\epsilon})$
on the square of the $\ell_2$-norm of the population gradient, i.e., $\| \nabla f(\w_R) \|^2$.  Using the fact that $\mathbb{E} \|\nabla f(\w_R) \| \leq \sqrt{\mathbb{E} \|\nabla f(\w_R) \|^{2}}$, the optimal rate of the $\ell_2$-norm of the population gradient i.e., $\| \nabla f(\w_R) \|$ is $\tilde O(\frac{\sqrt[4]{p}}{\sqrt{n\epsilon}})$. Our results and existing results \citep{basm14,waye17,waxu19} show that there is an additional factor $p$ in the bound caused by privacy compared with non-private case. Compared to the previous result $ O(\frac{\sqrt{p}}{\sqrt{n\epsilon}})$ in \cite{waxu19}, our rate shows improvement on the dependence dimension $p$. Note that with Polyak-Łojasiewicz condition \cite{polyak1963gradient,karimi2016linear}, i.e., 
$ f(\w_R) - f(\w^\star) \leq \kappa \|\nabla f(\w_R) \|^2$ for $\kappa >0$ with $\w^\star$ to be any population risk minimizer, which shows that the small gradient norm implies small population risk, one can genneralizes the Theorem \ref{thm:gen_bound_fb} to the population risk bound.
In terms of computational complexity, Algorithm \ref{algo: full batch} requires $O(\frac{n^2\epsilon}{\sqrt{p}})$ individual gradient computations for $O(n\epsilon/\sqrt{p})$ passes over $n$ samples, which is the same as the DP gradient algorithms in \cite{waxu19}.

\section{Empirical Risk Analysis}
\label{sec:erm}

In this section, we compare the generalization bound, i.e., the $\ell_2$-norm of the population gradient achieved based on uniform convergence and the bound given by our proof technique in Section \ref{sec:algorithm}. Hence, we first provide the empirical risk analysis of DPAGD, i.e., the bound on the $\ell_2$-norm of the empirical gradient. Then, using the empirical risk bound, we discuss the bound on the population gradient based on uniform convergence.

\begin{restatable}{theo}{theogdfb}
\label{thm:GD_fb}
({\bf DP GD}) Under the Assumption \ref{asmp: bounded_gradient} and \ref{asmp: smoothness}, for any $\epsilon, \delta >0$, DPAGD (Algorithm \ref{algo: full batch}) with $\phi_t(\tilde \g_1,...,\tilde \g_t) = \tilde \g_t$,  $\psi_t(\tilde \g_1,...,\tilde \g_t) = \I$, $\sigma^2$ be as in \eqref{eq: sigma_fb}, $\eta_t = \frac{1}{L}$, $ T = O\left(\frac{\sqrt{L} n \epsilon}{\sqrt{p \log (1 / \delta) }G}\right)$, $\lambda = 1$ and $\nu = 0$ achieves:
\begin{equation} 
\label{eq: opt_gd}
   \mathbb{E} \| \nabla \hat f(\w_R) \|^2 \leq O\left(\frac{\sqrt{L} G \sqrt{p \log (1 / \delta)}}{n \epsilon}\right),
\end{equation}
where $\w_R$ is is uniformly
sampled from $\{\w_1, \w_2, ...,\w_T\}$.
\end{restatable}

Theorem \ref{thm:GD_fb} shows that DP GD achieves the rate of $\frac{\sqrt[4]{p}}{\sqrt{n\epsilon}}$ on the $\ell_2$-norm of the empirical gradient. Actually, in this case, Algorithm \ref{algo: full batch} is exactly the Algorithm 4 in \cite{waye17} and we get the same result of the empirical gradient as in \cite{waye17}.

\begin{restatable}{theo}{theormspropfb}
\label{thm:rmsprop_fb}
({\bf DP RMSprop}) Under the Assumption \ref{asmp: bounded_gradient} and \ref{asmp: smoothness}, for any $\epsilon, \delta >0$, DPAGD (Algorithm \ref{algo: full batch}) with  $\phi_t(\tilde \g_1,...,\tilde \g_t) = \tilde \g_t$, and $ \psi_t = \left(1-\beta_{2}\right) \sum_{j=1}^{t} \beta_{2}^{t-j} \tilde \g_{j}^{2}$, $\sigma^2$ be as in \eqref{eq: sigma_fb}, $T=O\left(\frac{ n \epsilon}{\sqrt{p \log (1 / \delta) }G}\right)$, $\eta_t = \eta$, $\lambda > 0$ $\forall t \in [T]$, $\nu$, $\beta_2$ and $\eta$ are chosen such that: $\eta \leq \frac{\nu}{2L}$ and $1-\beta_{2} \leq \frac{\nu^{2}}{16 G^{2}}$
achieves:
\begin{equation} 
\label{eq: opt_rmsprop}
    \mathbb{E} \| \nabla \hat f(\w_R) \|^2 \leq O\left(\frac{ G^2 \sqrt{p \log (1 / \delta)}}{n \epsilon}\right),
\end{equation}
where $\w_R$ is uniformly
sampled from $\{\w_1, \w_2, ...,\w_T\}$.
\end{restatable}

Theorem \ref{thm:rmsprop_fb} shows that DP RMSprop achieves the same bound $\tilde O(\frac{\sqrt[4]{p}}{\sqrt{n\epsilon}})$ as DP GD.

\begin{restatable}{theo}{theoadamfb}
\label{thm:adam_fb}
({\bf DP Adam}) Under the Assumption \ref{asmp: bounded_gradient} and \ref{asmp: smoothness},  for any $\epsilon, \delta >0$ and $n \geq \frac{2 \ln (8 / \delta)}{\epsilon^{2}}$, for any $\beta >0$, DPAGD (Algorithm \ref{algo: full batch}) with $\sigma^2$ set to be as \eqref{eq: sigma_fb}, $\phi_t(\tilde \g_1,...,\tilde \g_t) = \left(1-\beta_{1}\right) \sum_{j=1}^{t} \beta_{1}^{t-j} \tilde \g_{j}$, and $ \psi_t = \left(1-\beta_{2}\right) \sum_{j=1}^{t} \beta_{2}^{t-j} \tilde \g_{j}^{2}$, $T = \frac{n\epsilon }{G \sqrt{p \ln(1/\delta)}}$ , step size $\eta_t = \eta$, $0<\beta_2<1$, $\lambda > 0$, $\beta_1$ and $\nu$ are chosen such that: \begin{small}
    $\eta \leq ({\sqrt{1/2+4\beta_1/(1-\beta_1)^2}-1/2})\frac{(1-\beta_1)^2}{\beta_1}\frac{\nu}{4L}$ \end{small}
satisfies, 
\begin{equation}
     \nr  \mathbb{E}\| \nabla \hat f( \w_R)\|^2  \leq  O\left( \frac{G^2\sqrt{p\ln (1/\delta)}\ln(n\sqrt{p}\epsilon/\beta)}{n\epsilon}\right),   
\end{equation}
with probability at least $1-\beta$. where $\w_R$ is uniformly sampled from $\{\w_1, \w_2, ...,\w_T\}$ and the expectation is over the draw of $\w_R$.
\end{restatable}

Theorem \ref{thm:adam_fb} shows that DP ADAM achieves the same bound $\tilde O(\frac{\sqrt[4]{p}}{\sqrt{n\epsilon}})$ as DP GD. Especially, based on the current optimization analysis \citep{chli18} of Adam that has a worse dependence on $p$, i.e., $\frac{\sqrt{p}}{\sqrt[4]{n}}$ over $n$ stochastic gradient computations/iterations.

From Theorem \ref{thm:rmsprop_fb}, Theorem \ref{thm:adam_fb} and Theorem \ref{thm:GD_fb}, we have $\mathbb{E} \|\nabla \hat f(\w_R) \| \leq \sqrt{\mathbb{E} \|\nabla \hat f(\w_R) \|^{2}} \leq O(\frac{\sqrt[4]{p}}{\sqrt{n\epsilon}})$.
The prior approach extends the bound on the empirical gradient $\| \nabla \hat f(\w)\|$ by using the uniform convergence of empirical gradient to population gradient, i.e., $ \mathrm{sup}_{\w} \|\nabla f(\w) -\nabla \hat f(\w) \| \leq O(\frac{\sqrt{p}}{\sqrt{n}})$ (Theorem 1 in \citep{meba16}). 
In Appendix \ref{app:lower_bound}, we provide a lower bound of gradient uniform convergence rate that matches rate of $\sqrt{p}/\sqrt{n}$. 
The lower bound suggests that, the uniform convergence approach, i.e., $\mathbb{E}\|\nabla f(\w_{R})- \hat f(\w_{R})\| \leq O(\frac{\sqrt{p}}{\sqrt{n}})$ and $\mathbb{E}\|\nabla f(\w_R) \| \leq 
O(\frac{\sqrt{p}}{\sqrt{n\epsilon}})$, fails to match our results in Theorem \ref{thm:gen_bound_fb}.

\section{Experiments}
\label{sec:experiments}

We empirically evaluate the performance of DP SGD, DP RMSprop and DP Adam \footnote{We  implemented the mini-batch version of DP GD, DP RMSprop and DP Adam in PyTorch based on this repository https://github.com/ChrisWaites/pyvacy.} for training various modern deep learning models. 
We consider MNIST image classification task \citep{lebo1998}.
After briefly discussing the experimental setup, we present experimental results.

{\bf Network Architecture and Datasets}: We focus on fully connected networks with ReLU activation of $2$ hidden layers with 128 nodes each layer. The MNIST dataset contains 60,000 black and white training images and 10,000 test examples, representing handwritten digits 0 to 9. Each image of size $28\times28$  is normalized by subtracting the mean and dividing the standard deviation of the training set and converted into a vector of size 784.

\textbf{Training and Hyper-parameter Setting:} Since optimization hyper-parameters affect the quality of solutions, and \cite{wiro17} find that the initial step size and the scheme of decaying step sizes have a marked impact on the performance, we follow the grid search method with search space $\{0.1, 0.01, 0.001\}$ to tune the step size.  For training, a fixed budget on the number of epochs i.e., 100 is assigned for 
the task. We decay the learning rate by 0.1 every 30 epochs.
The mini-batch size is set to be 128 for MNIST. Cross-entropy is used as our loss function throughout  experiments. We choose the settings achieving the lowest final training loss.  We repeat each experiments 5 times and report the mean and standard deviation of the accuracy on the training and test set.

{\bf Parameter of Differential Privacy.} Since the gradient bound $G$ is unknown for deep learning, we follow the gradient clipping method in \cite{abch16} to guarantee the privacy. We choose clip size to be $1.0$ for MNIST. We report results for three choices of the noise scale, i.e., $\sigma = \{2, 4, 8\}$ for MNIST. We follow the MA \cite{budl19} to calculate the accumulated privacy cost. Fixing $\delta = 10^{-5}$, the $\epsilon$ is $\{1.22, 0.57, 0.28\}$ for MNIST.

\begin{figure*}[t] 
\centering
 \subfigure[Training accuracy, $\epsilon = 0.28$]{
 \includegraphics[trim = 3mm 1mm 4mm 4mm, clip, width=0.30\textwidth]{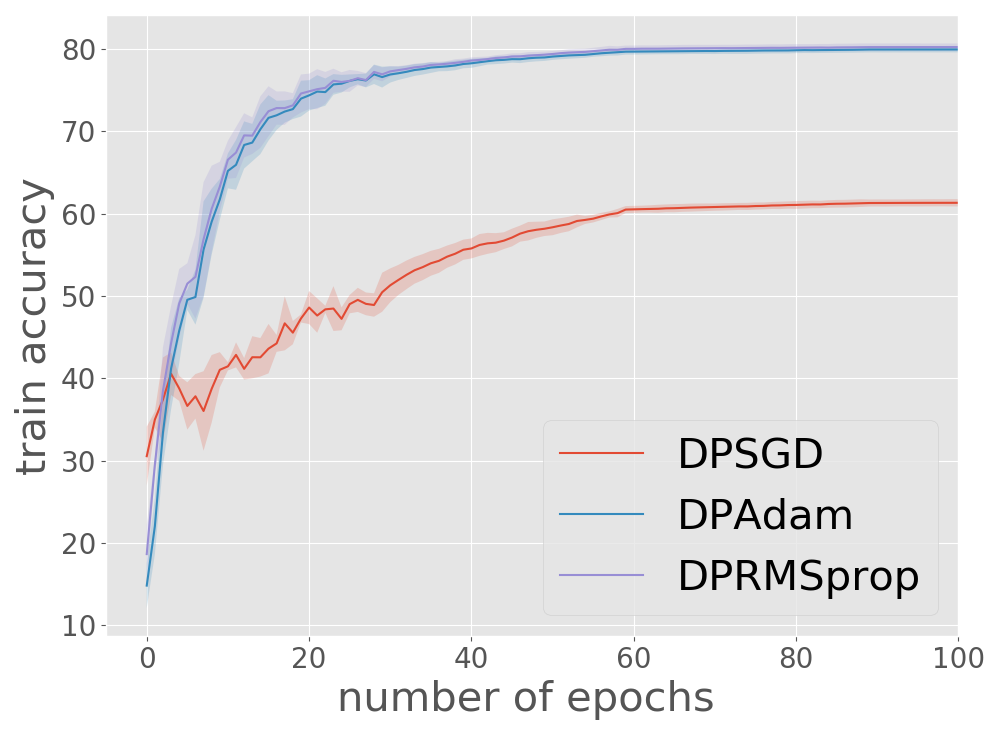}
 } 
  \subfigure[Training accuracy, $\epsilon = 0.57$]{
 \includegraphics[trim =3mm 1mm 4mm 4mm, clip, width=0.30\textwidth]{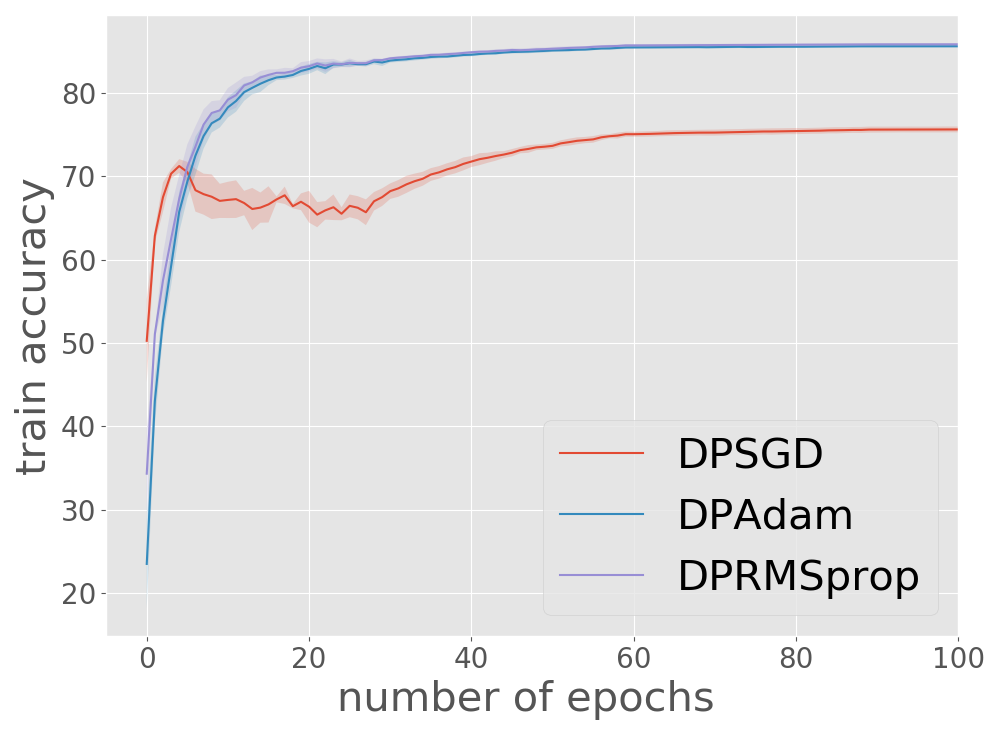}
 } 
 \subfigure[Training accuracy, $\epsilon = 1.22$]{
 \includegraphics[trim = 3mm 1mm 4mm 4mm, clip, width=0.30\textwidth]{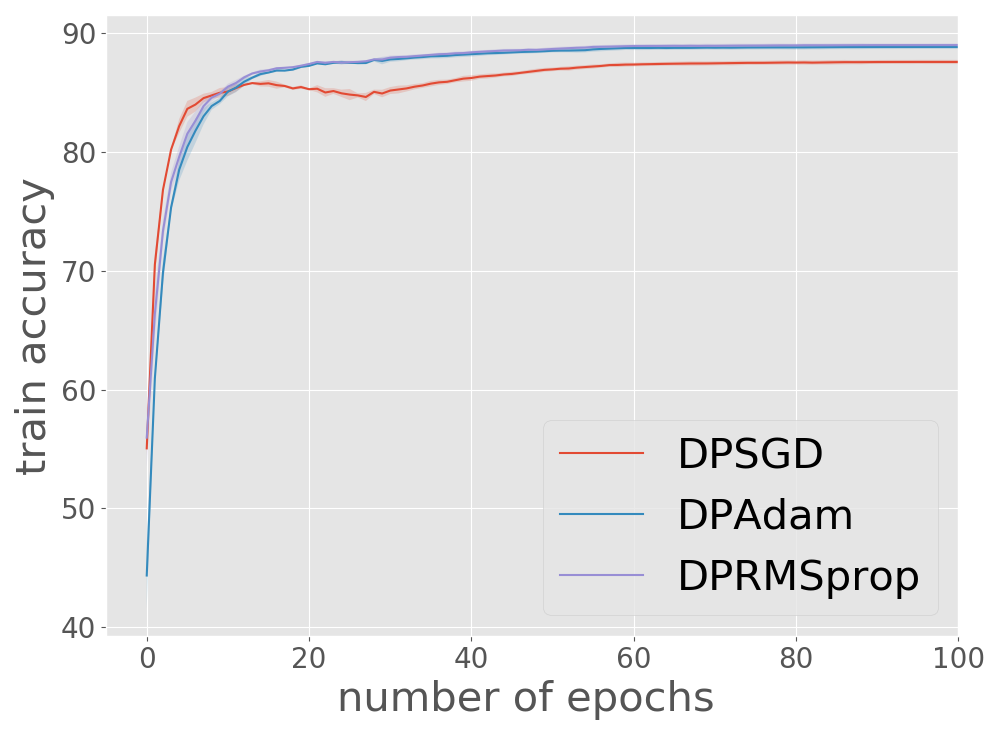}
 } \vspace{-3mm}
 \subfigure[Test accuracy, $\epsilon = 0.28$]{
 \includegraphics[trim = 3mm 1mm 4mm 4mm, clip, width=0.30\textwidth]{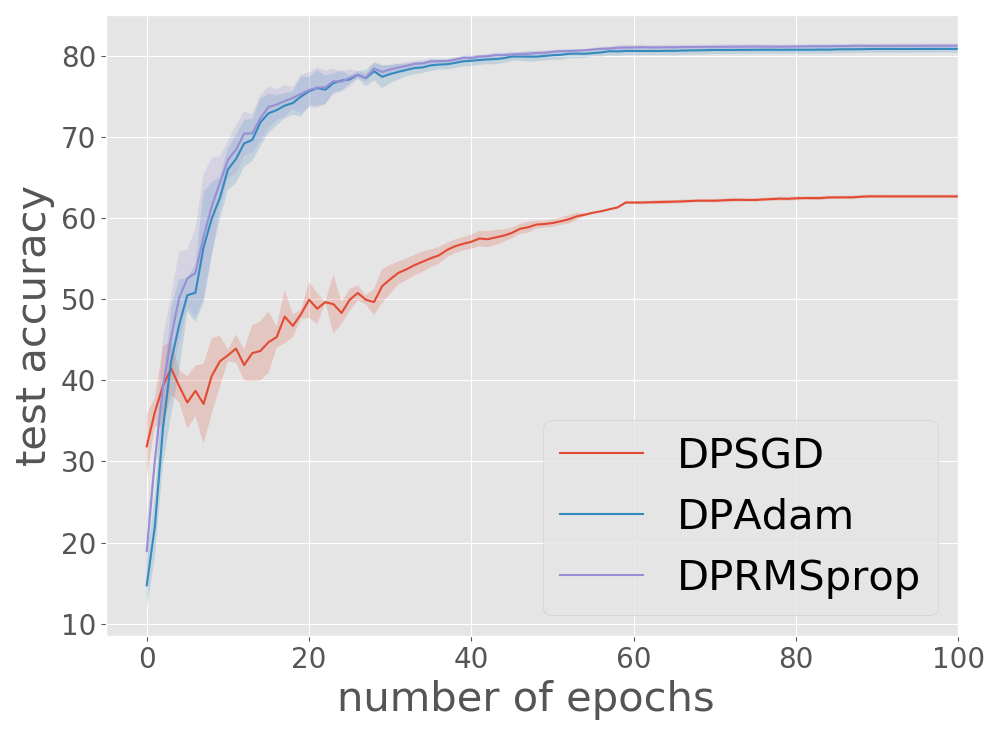}
 } 
  \subfigure[Test accuracy, $\epsilon = 0.57$]{
 \includegraphics[trim =3mm 1mm 4mm 4mm, clip, width=0.30\textwidth]{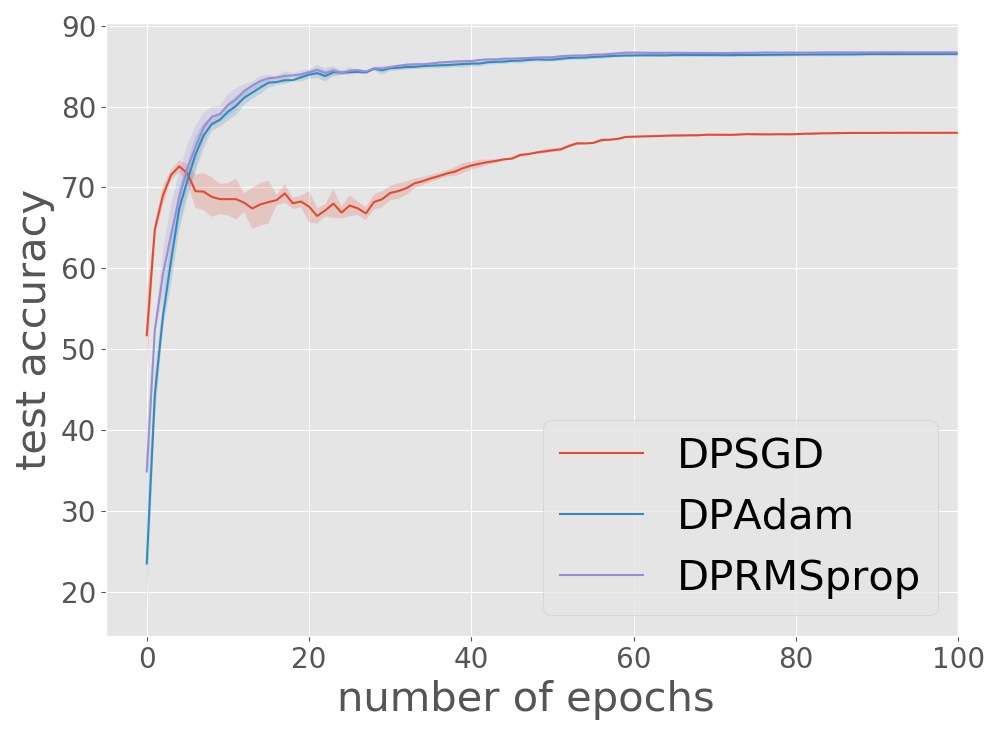}
 } 
 \subfigure[Test accuracy, $\epsilon = 1.22$]{
 \includegraphics[trim = 3mm 1mm 4mm 4mm, clip, width=0.30\textwidth]{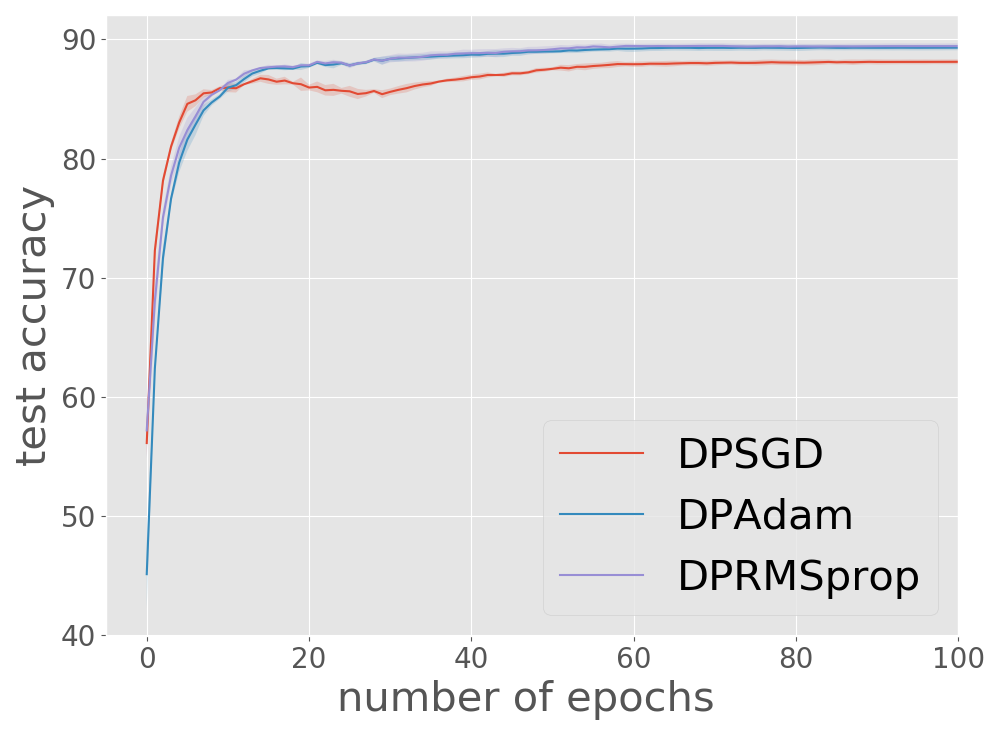}
 } 
\caption[]{Comparison of DP Adam, DP RMSprop and DP SGD on MNIST with $\epsilon = \{1.22, 0.57, 0.28\}$. (a-c) correspond to the training accuracy and (d-f) correspond to the test accuracy. The X-axis is the number of epochs, and the Y-axis is the train/test accuracy.  The adaptive gradient methods such as DP RMSprop and DP Adam achieve better training and test accuracy than DP SGD, especially for small $\epsilon$. 
}
\label{fig:mnist_eps}
\vspace*{-4mm}
\end{figure*}

{\bf Experimental Results.} The training accuracy and test accuracy for different level of privacy, i.e., $\epsilon$, are reported in Figure \ref{fig:mnist_eps}. Figure \ref{fig:mnist_eps} shows that adaptive methods such as DP Adam and DP RMSprop progress faster than DP SGD, especial for small privacy cost, i.e., $\epsilon = 0.28$.

\section{Conclusion}
\label{sec:conclusion}
In this paper, we study the differential private adaptive gradient descent algorithms for non-convex optimization. We provide population risk analysis using the generalization property of differential privacy itself and adaptive data analysis. We obtain a sharper bound on the $\ell_2$ norm of the population gradient by taking the advantages of generalization guarantee of differential privacy. 
We show that uniform convergence argument gives a worse generalization bound even if those algorithms obtain a better empirical gradient bound. Finally, we experimentally evaluate the proposed algorithms and show that DP adaptive gradient methods tend to outperform DP SGD for the task we consider.

\vspace{-0.2cm}
\section*{Acknowledgement}
The research was supported by NSF grants IIS-1908104, OAC-1934634, IIS-1563950, FAI 1939606, CMMI-172775, an ARO grant 73202-CS, an Amazon Research Award, a Google Faculty Research Award, and a Mozilla Research Grant. We would like to thank the Minnesota Super-computing Institute (MSI) for providing computational resources and support.



\bibliographystyle{plain}
\bibliography{references}

\newpage
\appendix

\section{Proofs for Section \ref{sec:gen_of_dp}}

\theodpfb*

\emph{Proof of Theorem \ref{thm:dp_analysis_fb}:}  Theorem \ref{thm:dp_analysis_fb} is a variant of Theorem 1 in \cite{abch16}. The proof follows by setting the sampling probability to be $1$ in the proof of Theorem 1 in \cite{abch16}.

\begin{restatable}{lemm}{lemmgenfb}
\label{lem:gen_adv_fb}
Set $\sigma^2$ in  DPAGD (Algorithm \ref{algo: full batch})  
to be as eq. \eqref{eq: sigma_fb} in Theorem \ref{thm:dp_analysis_fb}. Let $\w_t 
$ be the parameter generated at each iteration $t \in [T]$ and $\hat \g_t$ be the empirical gradient  such that $\hat \g_t = \mathbb{E}_{\z \in S} [\nabla \ell(\w_t, \z)]$. Then, for any $\tau >0$, $\rho> 0$, if the privacy cost of Algorithm \ref{algo: full batch} satisfies $\epsilon \leq \frac{\tau}{13}$, $\delta \leq \frac{\tau \rho}{26 \ln(26/\tau)}$ and the sample size $ n \geq \frac{2\ln(8/\delta)}{\epsilon^2}$, we have the following gradient concentration bound for $\hat \g_t$, i.e., $ \forall i\in [p]$ and  $\forall t \in [T]$,
\begin{equation}
    \mathbb{P}\left\{ |\hat \g_t^i - \g_t^i| \geq  \tau \right\} \leq \rho.
\end{equation}
\end{restatable}

\proof The proof follows by applying
Theorem 8 in \cite{dwfe2015a} to gradient descent, which shows that in order to achieve generalization error $\tau$ with probability $1-\rho$ for a $(\epsilon, \delta)$-differentially private algorithm (i.e., in order
to guarantee for every function $\phi_t$, $\forall t \in [T]$, we have $\mathbb{P}\left[\left|\mathcal{P}\left[\phi_t\right]-\mathcal{E}_{S'}\left[\phi_t\right]\right| \geq \tau\right] \leq \rho$), where $\mathcal{P}\left[\phi_t\right]$ is the population value, $\mathcal{E}_{S'}\left[\phi_t\right]$ is the empirical value evaluated on $S'$ and $\rho$ and $\tau$ are any positive constant, we can set the $\epsilon \leq \frac{\tau}{13}$, $\delta \leq \frac{\tau \rho}{26 \ln (26 / \tau)}$ and $|S'| \geq \frac{2\ln(8/\delta)}{\epsilon^2}$. In our instantiation, $\phi_t$ is 
the gradient computation function $\nabla \ell(\w_t,\z)$,  $\mathcal{P}\left[\phi_t\right]$ represents the population gradient $\g_t^i$, $\forall i\in [p]$, and $\mathcal{E}_{S'}\left[\phi_t\right]$ represents the sample gradient $\hat \g_t^i$, $\forall i\in [p]$. Thus we have $\mathbb{P}\left\{\left|\hat{\mathbf{g}}_{t}^{i}-\mathbf{g}_{t}^{i}\right| \geq \tau\right\} \leq \rho$ if $\epsilon \leq \frac{\tau}{13}, \delta \leq \frac{\tau \rho}{26 \ln (26 / \tau)}$ and  $n \geq \frac{2\ln(8/\delta)}{\epsilon^2}$.

Lemma \ref{lem:gen_adv_fb} gives the generalization error, i.e., gap between $\tilde \g_t$ and $\g_t$ for DPAGD. Using this result, we show that the noisy gradient $\tilde \g_t$ concentrates to $\g_t$ across all iterations, i.e., gradient estimation error $\|\tilde \g_t -\g_t\|$ can be bounded with high probability $\forall t \in [T]$ (Theorem \ref{thm: acc_basic_fb}).

\theoaccbasicfb*

\proof
The concentration bound can be dedecomposed into two parts:
\begin{equation}
    \mathbb{P}\left\{ \|\tilde \g_t - \g_t\| \geq \sqrt{d}\sigma(1+\mu)\right\} \leq
    \underbrace{ \mathbb{P}\left\{ \|\tilde \g_t - \hat \g_t\| \geq \sqrt{d}\sigma \mu\right\} }_{\text{$T_1$: empirical error}} + 
    \underbrace{\mathbb{P}\left\{ \|\hat \g_t - \g_t\| \geq \sqrt{d}\sigma \right\}}_{\text{$T_2$: generalization error}} 
\end{equation}
In the above inequality, there are two types of error we need to control. The first type of error, referred to as empirical error $T_1$, is the the deviation between the differentially
private estimate gradient $\tilde \g_t$ and the empirical gradient $\hat \g_t$. The second type of error, referred to as generalization error $T_2$, is the deviation
between the empirical gradient $\hat \g_t$ and the population gradient $\g_t$. 

The second term $T_2$ can be bounded thorough the generalization guarantee of differential privacy. Recall that from Lemma \ref{lem:gen_adv_fb}, under the condition that for any $\tau > 0$ and $\rho > 0$, $\epsilon \leq \frac{\tau}{13}$, $\delta \leq \frac{\tau \rho}{26 \ln(26/\sigma)}$ and the sample size $n \geq \frac{2\ln(8/\delta)}{\epsilon^2}$, we have  $\forall i\in [d]$ and  $ \forall t \in [T]$
\begin{equation}
    \mathbb{P}\left\{ |\hat \g_t^i - \g_t^i| \geq  \tau \right\} \leq \rho.
\end{equation}
Replace $\tau = \sigma$ and $\rho = 2\exp(-\mu^2/2)$, we have
\begin{equation}
    \mathbb{P}\left\{ |\hat \g_t^i - \g_t^i| \geq  \sigma \right\} \leq 2\exp(-\mu^2/2),
\end{equation}
under the condition that $\epsilon \leq \frac{\sigma}{13}$, $\delta \leq \frac{\sigma \exp(-\mu^2/2)}{13 \ln(26/\sigma)}$ and the sample size $n \geq \frac{2\ln(8/\delta)}{\epsilon^2}$. So that we have 
\begin{equation} \label{eq: gen1}
    \mathbb{P}\left\{ \|\hat \g_t - \g_t\| \geq  \sqrt{d}\sigma \right\} \leq \mathbb{P}\left\{ \|\hat \g_t - \g_t\|_\infty \geq  \sigma \right\} \leq d \mathbb{P}\left\{ |\hat \g_t^i - \g_t^i| \geq  \sigma \right\} \leq
    2p\exp(-\mu^2/2)
\end{equation}

Now we bound the second term $T_1$. Recall that $\tilde \g_t = \hat \g_t + \b_t$, where $\b_t$ is a noise vector drawn from Gaussian noise $\cN(0, \sigma^2\mathbb{I}_p)$. Using the tail bound of Gaussian random variable, we have
\begin{equation} \label{eq: acc1}
    \mathbb{P}\left\{ \|\tilde \g_t - \hat \g_t\| \geq  \sqrt{p}\sigma \mu \right\} \leq \mathbb{P} \left\{ \|\b_t\| \geq  \sqrt{p}\sigma \mu \right\} \leq \mathbb{P} \left\{ \|\b_t\|_\infty \geq  \sigma \mu \right\} \leq p \mathbb{P} \left\{ |\b_t^i| \geq \sigma \mu \right\} =  2p\exp(-\mu^2/2).
\end{equation}

The second inequality come from $\|\b_t\| \leq \sqrt{d}\|\b_t\|_\infty$.  
Combine \eqref{eq: gen1} and \eqref{eq: acc1}, we complete the proof.\qed

\section{Proofs for Section \ref{sec:con_rate}}

We restate the Theorem \ref{thm:gen_bound_fb} in the following three theorems, i.e., Theorem \ref{thm: gen_sgd_fb}, Theorem \ref{thm: gen_rmsprop_fb} and Theorem \ref{thm: gen_adam_fb} for DP SGD, DP RMSprop and DP Adam in in Section \ref{sec: proof_sgd_fb}, \ref{sec: proof_rmsprop_fb} and \ref{sec: proof_adam_fb} respectively. Then we provide the proof of them. Before that, we first give a simplified version of Theorem \ref{thm: acc_basic_fb} in Theorem \ref{thm: acc_simple}.

\begin{theo} 
\label{thm: acc_simple}
Assume $\sigma$, $\epsilon$ and $\delta$ are set to satisfy the conditions in Theorem \ref{thm: acc_basic_fb} such that $\epsilon \leq \frac{\sigma}{13}$, $\delta \leq \frac{\sigma \exp(-\mu^2/2)}{13 \ln(26/\sigma)}$ and $n \geq \frac{2\ln(8/\delta)}{\epsilon^2}$, for the noisy gradients $\tilde \g_1,...,  \tilde \g_T$ in Algorithm \ref{algo: full batch}, we have $\forall t \in [T]$ and any $\mu > 0$:
\begin{equation}
    \mathbb{P}\{\|\tilde \g_t - \g_t\| \geq \alpha \} \leq \xi,
\end{equation}
where $\alpha = \sqrt{p}\sigma(1+\mu)$ and $\xi = 4p\exp(-\mu^2/2)$.
\end{theo}

Theorem \ref{thm: acc_simple} uses $\alpha$ to present the concentration error $\sqrt{p}\sigma(1+\mu)$, and $\xi$ to present the probability $4p\exp(-\mu^2/2$. For simplicity, we  first refer to Theorem \ref{thm: acc_simple} and use $\alpha$ and $\xi$ in the following sections. Then we bring in $\alpha = \sqrt{p}\sigma(1+\mu)$ and $\xi = 4p\exp(-\mu^2/2)$ to complete the proof.

\subsection{Proof of Theorem \ref{thm: gen_sgd_fb}}
\label{sec: proof_sgd_fb}

Now we present the proof of Theorem \ref{thm: gen_sgd_fb}.

\begin{theo} \label{thm: gen_sgd_fb}
({\bf DP GD}) 
Under the Assumption \ref{asmp: bounded_gradient} and \ref{asmp: smoothness}, given training sample $S$ of size $n$, for any $\epsilon, \delta >0$ and $n \geq \frac{2 \ln (8 / \delta)}{\epsilon^{2}}$, for any $\beta >0$, Algorithm \ref{algo: full batch} with $\sigma$ set to be as \eqref{eq: sigma_fb},
$\phi_t(\tilde \g_1,...,\tilde \g_t) = \tilde \g_t$,  $\psi_t(\tilde \g_1,...,\tilde \g_t) = \I$, $\nu = 0$, $\lambda =1$, $T = \frac{n\epsilon \sqrt{L}}{G \sqrt{p \ln(1/\delta)}}$, and step size $\eta_t = \frac{1}{4L}$
satisfies, 
\begin{equation} 
 \mathbb{E} \| \nabla f(\w_R)\|^2  \leq    O\left( \frac{G\sqrt{pL\ln (1/\delta)}\ln(n\sqrt{p}\epsilon/\beta)}{n\epsilon}\right) 
\end{equation}
with probability at least $1-\beta$, where where $\w_R$ is uniformly sampled from $\{\w_1, \w_2, ...,\w_T\}$ and the expectation is over the draw of $\w_R$.
\end{theo}

\proof
Upon the choice of $\phi$ and $\psi$ in Theorem \ref{thm: gen_sgd_fb}, the update of Algorithm \ref{algo: full batch}  becomes:
\begin{equation}
    \w_{t+1} = \w_t - \eta_t \tilde \g_t,
\end{equation}
where we have $\tilde \g_t$ approximate population gradient $\g_t$ as $    \mathbb{P}\{\|\tilde \g_t - \g_t\| \geq \alpha \} \leq \xi$, and $\alpha$ and $\xi$ are given in Theorem \ref{thm: acc_simple}.

With Assumption \ref{asmp: smoothness} and the update of Algorithm \ref{algo: full batch} , let $\Delta_t = \tilde \g_t - \g_t$, we have
\begin{align}
\nr f(\w_{t+1})& \leq f(\w_t) + <\g_t, \w_{t+1} -\w_t> + \frac{L}{2}\eta_t^2 \|\tilde \g_t\|^2  \\ \nr
& \leq f(\w_t) -\eta_t \|\g_t\|^2 -\eta_t <\g_t, \Delta_t> + \frac{L\eta_t^2}{2} \|\tilde \g_t\|^2 \\ \nr
 & \leq f(\w_t) - \frac{\eta_t}{2} \|\g_t\|^2 + \frac{\eta_t}{2}\|\Delta_t\|^2 + L \eta_t^2 \left( \| \g_t\|^2 + \|\Delta_t\|^2 \right)\\
 & = f(\w_t) -\left(\frac{\eta_t}{2} - L\eta_t^2 \right)\left\|\g_t \right\|^2 + \left( \frac{\eta_t}{2} + L\eta_t^2\right) \|\Delta_t\|^2
\end{align}

Rearrange the above equation, apply Theorem \ref{thm: acc_simple} that  $\|\Delta_t\| \leq \alpha$ with probability at least $1-\xi$, then we have the following
\begin{equation}
    \begin{array}{cc}
        \left(\frac{\eta_t}{2} - L\eta_t^2 \right)\left\|\g_t \right\|^2 \leq &  f(\w_t) - f(\w_{t+1})  + \left( \frac{\eta_t}{2} + L\eta_t^2\right) \alpha^2\\
    \end{array}
\end{equation}
with probability at least $1- \xi$.

Set $\eta_t = \frac{L}{4}$ and sum the above equation over $t = 1,..., T$, with $f^\star = \ell^\star$ in Assumption \ref{asmp: smoothness}, we have
\begin{align}
   \frac{1}{T} \sum_{t=1}^T \frac{1}{16L} \| \g_t \|^2 \leq \frac{f(\w_1) -f^\star}{T} + \frac{3}{16L} \alpha^2
\end{align}
\begin{align}
\Rightarrow
   \frac{1}{T} \sum_{t=1}^T  \| \g_t \|^2 \leq \frac{16L \left(f(\w_1) -f^\star\right)}{T} + 3 \alpha^2
\end{align}
\begin{align} \label{eq: proof_sgd_mb}
\Rightarrow
     \mathbb{E}\| \nabla f(\w_R)\|^2 =  \frac{1}{T} \sum_{t=1}^T  \| \g_t \|^2 \leq \frac{16L \left(f(\w_1) -f^\star\right)}{T} + 3 \alpha^2
\end{align}
with probability at least $1-T\xi$.

Plugging in $\alpha = \sqrt{p}\sigma(1+\mu)$, $\xi = 4p\exp(-\mu^2/2)$
from Theorem \ref{thm: acc_simple}, $T = \frac{n\epsilon \sqrt{L}}{G \sqrt{p \ln(1/\delta)}}$, and $\sigma^2 = O\left(\frac{G^{2} T \ln \left(\frac{1}{\delta}\right)}{n^{2} \epsilon^{2}}\right)$ and setting $\mu = \sqrt{2\ln(4pT/\beta)}$ we have
\begin{align}
   \nr  \mathbb{E} \| \nabla f(\w_R)\|^2  &\leq \left(f(\w_1) -f^\star\right)\frac{16 G\sqrt{pL\ln (1/\delta)}}{n\epsilon} + 3 \frac{G\sqrt{pL\ln (1/\delta)} (1+\mu)^2}{n\epsilon} \\
     & \leq   O\left( \frac{G\sqrt{pL\ln (1/\delta)}\ln(n\sqrt{p}\epsilon/\beta)}{n\epsilon}\right)
\end{align}
with probability at least $1-\beta$. \qed


\subsection{Proof of Theorem \ref{thm: gen_rmsprop_fb}}
\label{sec: proof_rmsprop_fb}

\begin{theo} 
\label{thm: gen_rmsprop_fb} 
({\bf DP RMSprop}) Under the Assumption \ref{asmp: bounded_gradient} and \ref{asmp: smoothness}, given training sample $S$ of size $n$, for any $\epsilon, \delta >0$ and $n \geq \frac{2 \ln (8 / \delta)}{\epsilon^{2}}$, for any $\beta >0$,Algorithm \ref{algo: full batch} with $\sigma$ set to be as \eqref{eq: sigma_fb}, $\phi_t(\tilde \g_1,...,\tilde \g_t) = \tilde \g_t$, and $ \psi_t = \left(1-\beta_{2}\right) \sum_{j=1}^{t} \beta_{2}^{t-j} \tilde \g_{j}^{2}$, $T = \frac{n\epsilon }{G \sqrt{p \ln(1/\delta)}}$, step size $\eta_t = \eta$,  $0 < \beta_2 < 1$, $\lambda > 0$, parameters $\nu$ and $\eta$ are chosen such that: $\eta \leq \frac{\nu}{4L}$
satisfies, 
\begin{equation} 
\mathbb{E}\|\nabla f(\w_R)\|^2 \leq
   O\left( \frac{G \sqrt{p\ln (1/\delta)}\ln(n\sqrt{p}\epsilon/\beta)}{n\epsilon}\right) 
\end{equation}
with probability at least $1-\beta$,  where $\w_R$ is uniformly sampled from $\{\w_1, \w_2, ...,\w_T\}$ and the expectation is over the draw of $\w_R$.
\end{theo}

\proof
Recall that the update in Theorem \ref{thm: gen_rmsprop_fb} is the following
\begin{equation}
    \w_{t+1}^i  = \w_t^i -\eta_t \frac{\tilde \g_t^i}{\sqrt{\v_t^i}+\nu}.
\end{equation}

With Assumption \ref{asmp: smoothness} and the update of Algorithm \ref{algo: full batch}, let $\Delta_t = \tilde \g_t - \g_t$, we have
\begin{align}
\nr f(\w_{t+1})& \leq f(\w_t) + \left<\g_t, \w_{t+1}-\w_t\right> + \frac{L}{2} \left\|\w_{t+1}-\w_t \right\|^2\\ \nr
&= f(\w_t) -\eta_t \left<\g_t, \tilde \g_t/(\sqrt{\v_t} +\nu) \right> + \frac{L\eta_t^2}{2} \left\|\frac{\tilde \g_t}{(\sqrt{\v_t} +\nu)} \right\|^2\\ \nr
&= f(\w_t) -\eta_t \left<\g_t, \frac{\g_t +\Delta_t}{\sqrt{\v_t} +\nu} \right> + \frac{L\eta_t^2}{2}\left\|\frac{\g_t + \Delta_t}{\sqrt{\v_t} +\nu}\right\|^2 \\ \nr
&\leq f(\w_t) -\eta_t \left<\g_t, \frac{\g_t }{\sqrt{\v_t} +\nu}\right> -\eta_t \left<\g_t, \frac{\Delta_t }{\sqrt{\v_t} +\nu} \right> + L\eta_t^2\left(\left\|\frac{\g_t }{\sqrt{\v_t} +\nu}\right\|^2 + \left\|\frac{ \Delta_t}{\sqrt{\v_t} +\nu}\right\|^2   \right) \\ \nr
& = f(\w_t) -\eta_t \sum_{i=1}^d \frac{\left[\g_t\right]_i^2}{\sqrt{\v_t^i} +\nu} - \eta_t \sum_{i=1}^d \frac{\g_t^i \Delta_t^i}{\sqrt{\v_t^i} +\nu} +  L\eta_t^2\left(\sum_{i=1}^d\frac{[\g_t]_i^2 }{(\sqrt{\v_t^i} +\nu)^2} + \sum_{i=1}^d\frac{[\Delta_t]_i^2 }{(\sqrt{\v_t^i} +\nu)^2}
\right) \\ \nr
& \leq f(\w_t) -\eta_t \sum_{i=1}^d \frac{[\g_t]_i^2}{\sqrt{\v_t^i} +\nu} + \frac{\eta_t}{2}\sum_{i=1}^d \frac{[\g_t]_i^2 + [\Delta_t]_i^2}{\sqrt{\v_t^i} +  +\nu}+  \frac{L\eta_t^2}{\nu}\left(\sum_{i=1}^d\frac{[\g_t]_i^2 }{\sqrt{\v_t^i} +\nu} + \sum_{i=1}^d\frac{[\Delta_t]_i^2 }{\sqrt{\v_t^i} +\nu}
\right) \\
& = f(\w_t) - \left(\eta_t -\frac{\eta_t}{2} - \frac{L\eta_t^2}{\nu} \right)\sum_{i=1}^d\frac{[\g_t]_i^2 }{\sqrt{\v_t^i} +\nu} + \left(  \frac{\eta_t}{2} + \frac{L\eta_t^2}{\nu} \right)\sum_{i=1}^d\frac{[\Delta_t]_i^2 }{\sqrt{\v_t^i} +\nu}
\end{align}

Given the parameter setting  from the Theorem \ref{thm: gen_rmsprop_fb}, with $\eta_t = \eta$ we have the following condition hold
\begin{equation}
    \frac{L\eta_t}{\nu} \leq \frac{1}{4}.
\end{equation}
Then we obtain
\begin{align}
\nr f(\w_{t+1})& \leq f(\w_t) - \frac{\eta}{4} \sum_{i=1}^{d} \frac{\left[\mathbf{g}_{t}\right]_{i}^{2}}{\sqrt{\mathbf{v}_{t}^{i}}+\nu}+\frac{3 \eta}{4} \sum_{i=1}^{d} \frac{\left[\Delta_{t}\right]_{i}^{2}}{\sqrt{\v_t^i} + \nu} \\
& \leq f(\w_t) - \frac{\eta}{\sqrt{\lambda} + \nu} \|\g_t\|^2 + \frac{3\eta}{4\nu} \|\Delta_t\|^2
\end{align}
The second inequality follows from the fact that $0 \leq \v_t^i \leq \lambda$. Using the telescoping sum and rearranging the inequality,  we obtain
\begin{align}
\frac{\eta}{\sqrt{\lambda} + \nu} \sum_{t=1}^T \|\g_t\|^2  \leq f(\w_1) - f^\star + \frac{3\eta}{4\nu} \sum_{t=1}^T  \|\Delta_t\|^2
\end{align}

Multiplying with $\frac{\sqrt{\lambda} +\nu}{\eta T}$ on both sides and with the guarantee in Theorem 1 that $\|\Delta_t\| \leq \alpha$ with probability at least $1-\xi$,  we obtain
\begin{equation}
\frac{1}{T} \sum_{t=1}^T \|\g_t\|^2 \leq \left(\sqrt{\lambda}+\nu \right) \times \left(   \frac{ f(\w_1) - f^\star}{\eta T} + \frac{3 \alpha^2}{4\nu}\right)
\end{equation}
\begin{align} \label{eq: proof_rmsprop_mb}
\Rightarrow
     \mathbb{E}\| \nabla f(\w_R)\|^2 =  \frac{1}{T} \sum_{t=1}^T  \| \g_t \|^2 \leq \left(\sqrt{\lambda}+\nu \right) \times \left(   \frac{ f(\w_1) - f^\star}{\eta T} + \frac{3 \alpha^2}{4\nu}\right)
\end{align}
with probability at least $1-T\xi$.

Plugging in $\alpha = \sqrt{p}\sigma(1+\mu)$, $\xi = 4p\exp(-\mu^2/2)$
from Theorem \ref{thm: acc_simple}, $T = \frac{n\epsilon}{G \sqrt{p \ln(1/\delta)}}$, and $\sigma^2 = O\left(\frac{G^{2} T \ln \left(\frac{1}{\delta}\right)}{n^{2} \epsilon^{2}}\right)$ and setting $\mu = \sqrt{2\ln(4pT/\beta)}$ we have
\begin{align}
   \nr   \mathbb{E}\| \nabla f(\w_R)\|^2  &\leq \left(\sqrt{\lambda} + \nu \right) \left(\left(f(\w_1) -f^\star\right)\frac{ G\sqrt{p\ln (1/\delta)}}{\eta n\epsilon} + 3 \frac{G\sqrt{pL\ln (1/\delta)} (1+\mu)^2}{ n\epsilon}\right) \\
     & \leq   O\left( \frac{G\sqrt{p\ln (1/\delta)}\ln(n\sqrt{p}\epsilon/\beta)}{n\epsilon}\right) 
\end{align}
with probability at least $1-\beta$. \qed

\subsection{Proof of Theorem \ref{thm: gen_adam_fb}}
\label{sec: proof_adam_fb}

\begin{theo}
\label{thm: gen_adam_fb} 
({\bf DP Adam}) Under the Assumption \ref{asmp: bounded_gradient} and \ref{asmp: smoothness}, given training sample $S$ of size $n$, for any $\epsilon, \delta >0$ and $n \geq \frac{2 \ln (8 / \delta)}{\epsilon^{2}}$, for any $\beta >0$,Algorithm \ref{algo: full batch} with $\sigma$ set to be as \eqref{eq: sigma_fb}, $\phi_t(\tilde \g_1,...,\tilde \g_t) = \left(1-\beta_{1}\right) \sum_{j=1}^{t} \beta_{1}^{t-j} \tilde \g_{j}$, and $ \psi_t = \left(1-\beta_{2}\right) \sum_{j=1}^{t} \beta_{2}^{t-j} \tilde \g_{j}^{2}$, $T = \frac{n\epsilon }{G \sqrt{p \ln(1/\delta)}}$ , step size $\eta_t = \eta$, $0<\beta_2<1$, $\lambda > 0$, $\beta_1$ and $\nu$ are chosen such that: \begin{small}
    $\eta \leq ({\sqrt{1/2+4\beta_1/(1-\beta_1)^2}-1/2})\frac{(1-\beta_1)^2}{\beta_1}\frac{\nu}{4L}$ \end{small}
satisfies, 
\begin{equation} 
\mathbb{E}\|\nabla f(\w_R)\|^2 \leq
   O\left( \frac{G\sqrt{p\ln (1/\delta)}\ln(n\sqrt{p}\epsilon/\beta)}{n\epsilon}\right) 
\end{equation}
with probability at least $1-\beta$,  where $\w_R$ is uniformly sampled from $\{\w_1, \w_2, ...,\w_T\}$ and the expectation is over the draw of $\w_R$.
\end{theo}

\proof Recall that the Adam update rule is
\begin{equation}
    \w_{t+1}^i  = \w_t^i -\eta_t \frac{ \bm_t^i}{\sqrt{\v_t^i}+\nu},
\end{equation}
where $\bm_t = \psi_t = \left(1-\beta_{1}\right) \sum_{j=1}^{t} \beta_{1}^{t-j} \tilde \g_{j}$.

For the ease of presentation, we reload division operation for vectors. We let $\a/\b$ to  be element-wise division when both $\a$ and $\b$ are vectors. When $\a$ is a vector, $b$ is a scalar, we let $\a/b$ return a vector with each element in $\a$ divided by $b$. Multiplication, addition and subtraction are reloaded similarly.

With this notation, the update rule of Adam is rewritten as
\begin{equation}
    \w_{t+1}  = \w_t -\eta_t \frac{ \bm_t}{\sqrt{\v_t}+\nu}
\end{equation}

First, by smoothness assumption, we have
\begin{align}\label{eq: descent_adam}
\nr f(\w_{t+1})& \leq f(\w_t) + \lla \g_t, \w_{t+1} -\w_t \rra + \frac{L}{2} \|\w_{t+1} -\w_t\|^2  
\\ \nr & = f(\w_t) -\eta_t \frac{1}{\gamma_t} \lla \gamma_t \g_t, \frac{ \bm_t}{\sqrt{\v_t}+\nu} \rra + \frac{L}{2} \|\w_{t+1} -\w_t\|^2  
\\ \nr & = f(\w_t) -\frac{\eta_t}{2\gamma_t} \lp \left \|\frac{ \bm_t}{\sqrt{\sqrt{\v_t}+\nu}} \right \|^2 +  \left \|\frac{ \gamma_t \g_t}{\sqrt{\sqrt{\v_t}+\nu}} \right \|^2 -  \left \|\frac{ \gamma_t \g_t - \bm_t}{\sqrt{\sqrt{\v_t}+\nu}} \right \|^2 \rp  + \frac{L}{2} \|\w_{t+1} -\w_t\|^2  
\\  & = f(\w_t) -\frac{\eta_t}{2\gamma_t} \lp \left \|\frac{ \bm_t}{\sqrt{\sqrt{\v_t}+\nu}} \right \|^2 +  \left \|\frac{ \gamma_t \g_t}{\sqrt{\sqrt{\v_t}+\nu}} \right \|^2\rp   + \frac{\eta_t}{2\gamma_t} \underbrace {\left \|\frac{ \gamma_t \g_t - \bm_t}{\sqrt{\sqrt{\v_t}+\nu}} \right \|^2}_{U_1(t)}  + \frac{L}{2} \underbrace{\|\w_{t+1} -\w_t\|^2}_{U_2(t)}  
\end{align}
where we define $\gamma_t \triangleq 1-\beta_1^t $.

To proceed, we need to further upper-bound $U_1(t)$ and $U_2(t)$. We can first bound $U_2(t)$ as 
\begin{align}\label{eq: U2}
    \sum_{t=1} ^T U_2(t) &=  \sum_{t=1} ^T  \|\w_{t+1} -\w_t\|^2 = \sum_{t=1} ^T \eta^2 \left\|  \frac{m_t}{\sqrt{\v_t} + \nu}\right\|^2 \leq \sum_{t=1} ^T \eta^2 \frac{1}{\nu}\left\|  \frac{m_t}{\sqrt{\sqrt{\v_t} + \nu}}\right\|^2
\end{align}

For $U_1(t)$, we can rewrite it as
\begin{align}
    \nr U_1(t) = \left \|\frac{ \gamma_t \g_t - \bm_t}{\sqrt{\sqrt{\v_t}+\nu}} \right \|^2 & = \left \|\frac{ (1-\beta_1^t)\g_t - \left(1-\beta_{1}\right) \sum_{j=1}^{t} \beta_{1}^{t-j} \g_{t} + (\left(1-\beta_{1}\right) \sum_{j=1}^{t} \beta_{1}^{t-j}  \g_{t}- \left(1-\beta_{1}\right) \sum_{j=1}^{t} \beta_{1}^{t-j} \tilde \g_{j})}{\sqrt{\sqrt{\v_t}+\nu}} \right \|^2 
    \\ \nr & =    {\left \|\frac{  \left(1-\beta_{1}\right) \sum_{j=1}^{t} \beta_{1}^{t-j} (\g_t - \tilde \g_{j})}{\sqrt{\sqrt{\v_t}+\nu}} \right \|^2}
\end{align}
where the last equality is due to $\sum_{j=1}^t \beta_1 ^{t-j} = \frac{1-\beta_1^t}{1-\beta_1}$.

Now we have
\begin{align}
\nr U_1(t) &= \left \|\frac{  \left(1-\beta_{1}\right) \sum_{j=1}^{t} \beta_{1}^{t-j} (\g_t - \tilde \g_{j})}{\sqrt{\sqrt{\v_t}+\nu}} \right \|^2
 \leq  \frac{1}{{\nu}} \left \|{  \left(1-\beta_{1}\right) \sum_{j=1}^{t} \beta_{1}^{t-j} (\g_t - \tilde \g_{j})} \right \|^2 
 \\ \nr & = \frac{1}{\nu} \left(1-\beta_{1}\right)^2 \lla {   \sum_{j=1}^{t} \beta_{1}^{t-j} (\g_t - \tilde \g_{j})}, {  \sum_{k=1}^{t} \beta_{1}^{t-k} (\g_t - \tilde \g_{k})} \rra
 \\ \nr & =  \frac{1}{\nu} \left(1-\beta_{1}\right)^2  \sum_{j=1}^{t} \sum_{k=1}^{t} \beta_{1}^{t-j}   \beta_{1}^{t-k} \lla {     \g_t - \tilde \g_{j}}, {   \g_t - \tilde \g_{k}} \rra
  \\ \nr & \leq   \frac{1}{\nu} \left(1-\beta_{1}\right)^2  \sum_{j=1}^{t} \sum_{k=1}^{t} \beta_{1}^{t-j}   \beta_{1}^{t-k} \frac{1}{2}  \lp \| \g_t - \tilde \g_{j}\|^2 + \|  \g_t - \tilde \g_{k}\|^2 \rp
    \\ \nr & =   \frac{1}{\nu} \left(1-\beta_{1}\right)^2  \sum_{j=1}^{t} \sum_{k=1}^{t} \beta_{1}^{t-j}   \beta_{1}^{t-k}  \| \g_t - \tilde \g_{j}\|^2
    \\ \nr & =   \frac{1}{\nu} \left(1-\beta_{1}\right)^2  \sum_{j=1}^{t}  \beta_{1}^{t-j}    \| \g_t - \tilde \g_{j}\|^2  \sum_{k=1}^{t}  \beta_{1}^{t-k}
    \\ & \leq   \frac{1}{\nu} \left(1-\beta_{1}\right)  \sum_{j=1}^{t}  \beta_{1}^{t-j}    \| \g_t - \tilde \g_{j}\|^2 
\end{align}

By the smoothness assumption, we further have
\begin{align}\label{eq: U1}
\nr U_1(t) & \leq    \frac{1}{\nu} \left(1-\beta_{1}\right)  \sum_{j=1}^{t}  \beta_{1}^{t-j}    \| \g_t - \tilde \g_{j}\|^2
\\ \nr & \leq \frac{1}{\nu} \left(1-\beta_{1}\right)  \sum_{j=1}^{t}  \beta_{1}^{t-j}     2(\| \g_t -  \g_{j} \|^2 + \| \g_j - \tilde \g_{j} \|^2 )
\\  & \leq \frac{2L^2}{\nu} \left(1-\beta_{1}\right)  \underbrace{\sum_{j=1}^{t}  \beta_{1}^{t-j} \| \w_t -  \w_{j} \|^2}_{U_3(t)}
+ \frac{1}{\nu} \left(1-\beta_{1}\right)  \underbrace{\sum_{j=1}^{t}  \beta_{1}^{t-j}  \| \g_j - \tilde \g_{j} \|^2}_{U_4(t)} 
\end{align}

For $U_3(t)$, when $\eta_t = \eta$ we have 
\begin{align}\label{eq: U3}
  \nr  \sum_{t=1} ^T U_3(t) &= \sum_{t=1}^T \sum_{j=1}^{t}  \beta_{1}^{t-j} \|  \w_t -  \w_{j} \|^2 = \sum_{t=1}^T \sum_{k=0}^{t-1}  \beta_{1}^{k} \|  \w_t -  \w_{t-k} \|^2
  \\ \nr & = \sum_{t=1}^T \sum_{k=0}^{t-1}  \beta_{1}^{k} \eta^2 \left\| \sum_{l=t-k}^{t-1}  \frac{m_l}{\sqrt{v_l} + \nu}\right\|^2
  \\ \nr & \leq \frac{2\beta_1}{(1-\beta_1)^3} \eta^2 \sum_{t=1}^{T-1} \left\|  \frac{m_t}{\sqrt{\v_t} + \nu}\right\|^2
  \\  & \leq \frac{2\beta_1}{(1-\beta_1)^3} \eta^2 \frac{1}{{\nu}} \sum_{t=1}^{T-1} \left\|  \frac{m_t}{\sqrt{\sqrt{\v_t} + \nu}}\right\|^2
\end{align}
    where the first inequality is by applying Lemma \ref{lem: adam_key_seq}.

Now we bound  $U_4(t)$ using Theorem \ref{thm: acc_simple} as 
\begin{align}\label{eq: U4}
    \sum_{t=1}^T U_4(t) = \sum_{t=1}^T \sum_{j=1}^{t}  \beta_{1}^{t-j}  \| \g_j - \tilde \g_{j} \|^2 \leq \sum_{t=1}^T \sum_{j=1}^{t}  \beta_{1}^{t-j} \alpha^2 \leq \frac{T}{1-\beta_1} \alpha^2
\end{align}
with probability at least $1-T\xi$.

Now, sum  \eqref{eq: descent_adam} with $t$ from 1 to $T$ and substitute into \eqref{eq: U1}, \eqref{eq: U2}, \eqref{eq: U3} and \eqref{eq: U4} with some rearrangement, we get
\begin{align}\label{eq: adam_opt}
\nr \sum_{t=1}^T \frac{\eta}{2\gamma_t} \lp \left \|\frac{ \bm_t}{\sqrt{\sqrt{\v_t}+\nu}} \right \|^2 +  \left \|\frac{ \gamma_t \g_t}{\sqrt{\sqrt{\v_t}+\nu}} \right \|^2\rp  & \leq  f(\w_1) -  f(\w_{T+1})   + \frac{L}{2\nu}  \eta^2  \sum_{t=1} ^T \left\|  \frac{m_t}{\sqrt{\sqrt{\v_t} + \nu}}\right\|^2 
\\  & + {\eta^3 }\frac{2L^2}{\nu^2}  \frac{\beta_1}{(1-\beta_1)^2}  \sum_{t=1}^{T-1} \left\|  \frac{m_t}{\sqrt{\sqrt{\v_t} + \nu}}\right\|^2
+ \frac{\eta}{2(1-\beta_1)}\frac{1}{\nu}   {T} \alpha^2 
\end{align}

Merging similar terms, we have
\begin{align}
 \sum_{t=1}^T \frac{\eta}{2\gamma_t} \left \|\frac{ \gamma_t \g_t}{\sqrt{\sqrt{\v_t}+\nu}} \right \|^2 & \leq  f(\w_1) -  f(\w_{T+1})   + \frac{\eta}{2(1-\beta_1)}\frac{1}{\nu}   {T} \alpha^2   
\\ \nr & +\lp {\eta^3 }\frac{2L^2}{\nu^2}  \frac{\beta_1}{(1-\beta_1)^2}+ \eta^2\frac{L}{2\nu} -\frac{\eta}{2} \rp \sum_{t=1}^{T} \frac{1}{\gamma_t}  \left\|  \frac{m_t}{\sqrt{\sqrt{\v_t} + \nu}}\right\|^2
\end{align}

When $\lp {\eta^3 }\frac{2L^2}{\nu^2}  \frac{\beta_1}{(1-\beta_1)^2}+ \eta^2\frac{L}{2\nu} -\frac{\eta}{2} \rp \leq 0$, i.e. $\eta \leq ({\sqrt{1/2+4\beta_1/(1-\beta_1)^2}-1/2})\frac{(1-\beta_1)^2}{4\beta_1}\frac{\nu}{L}$  we can further simplify the above inequality as 
\begin{align}\label{eq: adam_optw}
\frac{(1-\beta_1 )}{\sqrt{\lambda}+\nu}\sum_{t=1}^T \frac{\eta}{2}  \left \|{  \g_t} \right \|^2 \leq \sum_{t=1}^T \frac{\eta}{2\gamma_t}  \left \|\frac{ \gamma_t \g_t}{\sqrt{\sqrt{\v_t}+\nu}} \right \|^2  & \leq  f(\w_1) -  f(\w_{T+1})   + \frac{\eta}{2(1-\beta_1)}\frac{1}{\nu}   {T} \alpha^2  
\end{align}
where the first inequality is due to the fact that $0 \leq \v_t^i \leq \lambda$.

Rearranging, we have
\begin{align}
    \nr\frac{1}{T}\sum_{t=1}^T   \left \|{  \g_t} \right \|^2  \leq \frac{\sqrt{\lambda}+\nu}{1-\beta_1} \lp \frac{2}{\eta T}(f(\w_1) -  f(\w_{T+1}))   + \frac{1}{(1-\beta_1)\nu}    \alpha^2 \rp
\end{align}

Pick $R$ uniformly randomly from 1 to $T$, we know
\begin{align}\label{eq: adam_final}
    &\mathbb{E}\| \nabla f(\w_R)\|^2  \leq \frac{\sqrt{\lambda}+\nu}{1-\beta_1} \lp \frac{2}{\eta T}(f(\w_1) -  f(\w_{T+1}))   + \frac{1}{(1-\beta_1)\nu}    \alpha^2 \rp
\end{align}

with probability at least $1-T\xi$.

Plugging in $\alpha = \sqrt{p}\sigma(1+\mu)$, $\xi = 4p\exp(-\mu^2/2)$
from Theorem \ref{thm: acc_simple}, $T = \frac{n\epsilon}{G \sqrt{p \ln(1/\delta)}}$, and $\sigma^2 = O\left(\frac{G^{2} T \ln \left(\frac{1}{\delta}\right)}{n^{2} \epsilon^{2}}\right)$ and setting $\mu = \sqrt{2\ln(4pT/\beta)}$, with $f^\star = \ell^\star$ in Assumption \ref{asmp: smoothness}, we have
\begin{align}
   \nr   \mathbb{E}\| \nabla f(\w_R)\|^2  &\leq \left(\sqrt{\lambda} + \nu \right) \left(\left(f(\w_1) -f^\star\right)\frac{ G\sqrt{p\ln (1/\delta)}}{\eta n\epsilon} + 3 \frac{G\sqrt{pL\ln (1/\delta)} (1+\mu)^2}{ n\epsilon}\right) \\
     & \leq   O\left( \frac{G\sqrt{p\ln (1/\delta)}\ln(n\sqrt{p}\epsilon/\beta)}{n\epsilon}\right) 
\end{align}
with probability at least $1-\beta$. \qed

\begin{lemm}\label{lem: adam_key_seq}
For any $T \geq 1$, $  0 < \beta_1 < 1$ and $b_t$ we have
\begin{align} 
   \sum_{t=1}^T \sum_{k=0}^{t-1}  \beta_{1}^{k}  \left\| \sum_{l=t-k}^{t-1}  b_l\right\|^2 \leq \frac{2\beta_1}{(1-\beta_1)^3} \sum_{t=1}^{T-1} \left\|  b_t\right\|^2 
\end{align}
\end{lemm}
\allowdisplaybreaks[2]
\proof: The proof consists of a series of algebraic manipulations as follows.
\begin{align}
    \nr &\sum_{t=1}^T \sum_{k=0}^{t-1}  \beta_{1}^{k}  \left\| \sum_{l=t-k}^{t-1}  b_l\right\|^2 \\
    \nr \leq & \sum_{t=1}^T \sum_{k=0}^{t-1}  \beta_{1}^{k} k \sum_{l=t-k}^{t-1}   \left\|  b_l\right\|^2 \\
    \nr = & \sum_{t=1}^T \sum_{l=1}^{t-1} \sum_{k=t-l}^{t-1}  k \beta_{1}^{k}    \left\|  b_l\right\|^2 \\
    \nr = & \sum_{t=1}^T \sum_{l=1}^{t-1} \left\|  b_l\right\|^2 \sum_{k=t-l}^{t-1} \sum_{o=1}^k  \beta_{1}^{k}  \\
    \nr = & \sum_{t=1}^T \sum_{l=1}^{t-1} \left\|  b_l\right\|^2 \sum_{o=1}^{t-1 }\sum_{k=\max(o,t-l)}^{t-1}   \beta_{1}^{k} \\
     \nr \leq  & \sum_{t=1}^T \sum_{l=1}^{t-1} \left\|  b_l\right\|^2 \sum_{o=1}^{t-1 }\frac{1}{1-\beta_1}   \beta_{1}^{\max(o,t-l)} \\
     \nr \leq  & \frac{1}{1-\beta_1} \sum_{t=1}^T \sum_{l=1}^{t-1} \left\|  b_l\right\|^2  \lp (t-l)  \beta_{1}^{t-l } + \frac{1}{1-\beta_1} \beta_{1}^{t-l+1}  \rp\\
     \nr =  & \frac{1}{1-\beta_1} \sum_{l=1}^{T-1}   \left\|  b_l\right\|^2 \sum_{t=l+1}^T \lp (t-l)  \beta_{1}^{t-l } + \frac{1}{1-\beta_1} \beta_{1}^{t-l+1}  \rp\\
     \nr \overset{(a)}{\leq}   & \frac{1}{1-\beta_1} \sum_{l=1}^{T-1}   \left\|  b_l\right\|^2  \lp  \frac{\beta_1^2}{(1-\beta_1)^2} + \sum_{r=1}^{T-l} r  \beta_{1}^{r }   \rp\\
     \nr \overset{(b)}{\leq}   & \frac{1}{1-\beta_1} \sum_{l=1}^{T-1}   \left\|  b_l\right\|^2  \lp  \frac{\beta_1^2}{(1-\beta_1)^2} + \frac{\beta_1}{(1-\beta_1)^2}     \rp\\
     \nr \leq   & 2\frac{\beta_1}{(1-\beta_1)^3} \sum_{l=1}^{T-1}   \left\|  b_l\right\|^2  \\
\end{align}
where (a) is by introducing $r =t-l$ and (b) is due to $\sum_{r=1}^{T-l} r  \beta_{1}^{r } = \sum_{r=1}^{T-l} \sum_{q=1}^r  \beta_{1}^{r } =  \sum_{q=1}^{T-l} \sum_{r=q}^{T-l}   \beta_{1}^{r }   \leq \frac{1}{1-\beta_1} \sum_{q=1}^{T-1}  \beta_{1}^{q } \leq \frac{\beta}{(1-\beta_1)^2}$.
\qed
\allowdisplaybreaks[0]

\section{Proofs for Section \ref{sec:erm}}
We present the proof of Theorem \ref{thm:GD_fb}, Theorem \ref{thm:rmsprop_fb} and Theorem \ref{thm:adam_fb} in Section \ref{sec:proof_opt_gd}, Section \ref{sec:proof_opt_rmsprop} and Section \ref{sec:proof_opt_adam} respectively.

\subsection{Proof of Theorem \ref{thm:GD_fb}}
\label{sec:proof_opt_gd}

\theogdfb*

\proof
Let $\hat \g_t = \mathbb{E}_{\z \in S} [\nabla \ell(\w_t,\z)]$ denotes the full-batch gradient at iteration $t$. We have  $\hat \g_t = \nabla \hat f(\w_t)$.

Using this notation, we have the update of Algorithm \ref{algo: full batch} in Theorem \ref{thm:GD_fb} is 
\begin{equation}
    \w_{t+1} = \w_t - \eta_t \tilde \g_t
\end{equation}
where $\tilde \g_t = \hat \g_t + \b_t$.

By descent lemma:
\begin{align}
\nr \mathbb{E}_t [ \hat f(\w_{t+1})] & \leq \hat f(\w_t) + \mathbb{E}_t \left[<\nabla \hat f(\w_t), \w_{t+1} -\w_t> \right] + \frac{L}{2}\eta_t^2 \mathbb{E}_t \left[\|\tilde \g_t\|^2 \right]  \\ \nr
& = \hat f(\w_t) -\eta_t \mathbb{E}_t <\nabla \hat f(\w_t), \hat \g_t + \b_t> + \frac{L}{2}\eta_t^2 \mathbb{E}_t \left[\| \hat \g_t + \b_t\|^2 \right] \\
&\leq \hat f(\w_t) - \left(\eta_t + \frac{L}{2} \eta_t^2 \right) \|\nabla \hat f(\w_t) \|^2 + \frac{L}{2} \eta_t^2 p\sigma^2
\end{align}

Take $\eta_t = \frac{1}{L}$ we have
\begin{equation}
   \mathbb{E}_t [\hat f(\w_{t+1})] \leq \hat f(\w_t) - \frac{1}{2L} \|\nabla \hat f(\w_t) \|^2 + \frac{1}{2L} p\sigma^2
\end{equation}

Using telescoping sum and
rearranging the inequality, with $f^\star = \ell^\star$ in Assumption \ref{asmp: smoothness}, we obtain 
\begin{equation}
    \mathbb{E} \| \nabla \hat f(\w_R) \|^2 = \frac{1}{T} \sum_{t=1}^T \mathbb{E} \| \hat \nabla f(w_t) \|^2 \leq \frac{L\left( \hat f(\w_1) -  f^\star\right)}{T} + p\sigma^2
\end{equation}

Plugging in $ T = O\left(\frac{\sqrt{L} n \epsilon}{\sqrt{p \log (1 / \delta) }G}\right)$ and $\sigma^2 = O\left(\frac{G^{2} T \ln \left(\frac{1}{\delta}\right)}{n^{2} \epsilon^{2}}\right)$ achieves:
\begin{equation} 
   \mathbb{E} \| \nabla \hat f(\w_R) \|^2 \leq O\left(\frac{\sqrt{L} G \sqrt{p \log (1 / \delta)}}{n \epsilon}\right),
\end{equation}
where $\w_R$ is is uniformly
sampled from $\{\w_1, \w_2, ...,\w_T\}$.
\qed

\subsection{Proof of Theorem \ref{thm:rmsprop_fb}}
\label{sec:proof_opt_rmsprop}
We restate the Theorem \ref{thm:rmsprop_fb} here for convenience.
\theormspropfb*

\proof
Let $\hat \g_t = \mathbb{E}_{\z \in S} [\nabla \ell(\w_t,\z)]$ denotes the full-batch gradient at iteration $t$. We have  $\hat \g_t = \nabla \hat f(\w_t)$.

Using this notation, we have the update of Algorithm \ref{algo: full batch} in Theorem \ref{thm:rmsprop_fb} is 
\begin{equation}
    \w_{t+1} = \w_t - \eta_t \tilde  \g_t /(\sqrt{\v_t} + \nu) ,
\end{equation}
where $\tilde \g_t = \hat \g_t + \b_t$ and $\v_t = \left(1-\beta_{2}\right) \sum_{i=j}^{t} \beta_{2}^{t-j} \tilde \g_{j}^{2}$.

By descent lemma, we have
\begin{small}
\begin{align}
\nr \mathbb{E}_t [\hat f(\w_{t+1})]& \leq \hat f(\w_t) + \mathbb{E}_t \left<\hat \g_t,  \w_{t+1}-\w_t\right> + \frac{L}{2} \mathbb{E}_t \left\|\w_{t+1}-\w_t  \right\|^2\\ \nr
&= \hat f(\w_t) -\eta_t \mathbb{E}_t \left<\hat \g_t, \tilde \g_t/(\sqrt{\v_t} +\nu) \right> + \frac{L\eta_t^2}{2} \mathbb{E}_t \left\|\frac{\tilde \g_t}{(\sqrt{\v_t} +\nu)} \right\|^2\\ \nr
&= \hat f(\w_t)  -\eta_t \sum_{i =1}^{p} \left(\hat \g_t^i \times \mathbb{E}_t \left[ \frac{\hat \g_t^i + \b_t^i}{\sqrt{\beta_2 \v_{t-1}^i} + \nu} + \frac{\hat \g_t^i + \b_t^i}{\sqrt{ \v_{t}^i} + \nu} - \frac{\hat \g_t^i + \b_t^i}{\sqrt{\beta_2 \v_{t-1}^i} + \nu}\right]    \right) + \frac{L\eta_t^2}{2} \mathbb{E}_t \left\|\frac{ \hat \g_t + \b_t}{(\sqrt{\v_t} +\nu)} \right\|^2\\ \nr
& = \hat f(\w_t)  -\eta_t \sum_{i =1}^{p} \left( \hat \g_t^i \times  \left[ \frac{\hat \g_t^i }{\sqrt{\beta_2 \v_{t-1}^i} + \nu} + \mathbb{E}_t \left[\frac{\hat \g_t^i + \b_t^i}{\sqrt{ \v_{t}^i} + \nu} - \frac{\hat \g_t^i + \b_t^i}{\sqrt{\beta_2 \v_{t-1}^i} + \nu}\right]\right]    \right) + \frac{L\eta_t^2}{2} \mathbb{E}_t \left\|\frac{ \hat \g_t + \b_t}{(\sqrt{\v_t} +\nu)} \right\|^2\\ 
& \leq \hat f(\w_t) - \eta_t \sum_{i =1}^{p}\frac{[\hat \g_t]_i^2}{\sqrt{\beta_2 \v_{t-1}^i} + \nu} + \eta_t \sum_{i =1}^{p} |\hat \g_t^i|\left| \mathbb{E}_t \underbrace{\left[\frac{\hat \g_t^i + \b_t^i}{\sqrt{ \v_{t}^i} + \nu} - \frac{\hat \g_t^i + \b_t^i}{\sqrt{\beta_2 \v_{t-1}^i} + \nu}\right]}_{T_1}  \right| + \frac{L\eta_t^2}{2} \mathbb{E}_t \left\|\frac{ \hat \g_t + \b_t}{(\sqrt{\v_t} +\nu)} \right\|^2
\label{eq: rmsprop_full}
\end{align}
\end{small}

The forth equality follows from the fact that $\b_t$ and $\g_t$ are independent of $\v_{t-1}$ conditioned on the release of the past parameters and noise at time step $t$. Now we found $T_1$:
\begin{small}
\begin{align} 
\nr T_{1} & = \frac{\hat \g_t^i + \b_t^i}{\sqrt{ \v_{t}^i} + \nu} - \frac{\hat \g_t^i + \b_t^i}{\sqrt{\beta_2 \v_{t-1}^i} + \nu} \\ \nr
&\leq\left|\hat \g_t^i + \b_t^i\right| \times\left|\frac{1}{\sqrt{\v_t^i}+\nu}-\frac{1}{\sqrt{\beta_{2} \v_{t-1}^i}+\nu}\right| \\ 
\nr &=\frac{\left|\hat \g_t^i + \b_t^i\right|}{(\sqrt{\v_t^i}+\nu)(\sqrt{\beta_{2} \v_{t-1}^i}+\nu)} \times\left|\frac{\v_t^i-\beta_{2} \v_{t-1}^i}{\sqrt{\v_t^i}+\sqrt{\beta_{2} \v_{t-1}^i}}\right| \\ 
\nr &=\frac{\left|\hat \g_t^i + \b_t^i\right|}{(\sqrt{\v_t^i}+\nu)(\sqrt{\beta_{2} \v_{t-1}^i}+\nu)} \times \frac{\left(1-\beta_{2}\right) (\hat \g_t^i + \b_t^i)^{2}}{\sqrt{\v_t^i}+\sqrt{\beta_{2} \v_{t-1}^i}} \\
\nr & = \frac{\left|\hat \g_t^i + \b_t^i\right|}{(\sqrt{\v_t^i}+\nu)(\sqrt{\beta_{2} \v_{t-1}^i}+\nu)}  \times
\frac{\left(1-\beta_{2}\right) (\hat \g_t^i + \b_t^i)^{2}}{\sqrt{\beta_{2} \v_{t-1}^i + (1-\beta_2)(\hat \g_t^i + \b_t^i)^{2}}+\sqrt{\beta_{2} \v_{t-1}^i}} \\
\nr & \leq \frac{1}{(\sqrt{\v_t^i}+\nu)(\sqrt{\beta_{2} \v_{t-1}^i}+\nu)} \times \sqrt{(1-\beta_2)}(\hat \g_t^i + \b_t^i)^{2} \\
& \leq \frac{\sqrt{(1-\beta_2)}(\hat \g_t^i + \b_t^i)^{2}}{(\sqrt{\beta_{2} \v_{t-1}^i}+\nu)\nu}
\end{align}
\end{small}
Here, the last inequality is obtained by dropping $\v_t^i$ from the denominator to obtain an upper bound. The second inequality is  due to the fact that
\begin{equation}
\frac{\left|\hat \g_t^i + \b_t^i\right|}{\sqrt{\beta_{2} \v_{t-1}^i + (1-\beta_2)(\hat \g_t^i + \b_t^i)^{2}}+\sqrt{\beta_{2} \v_{t-1}^i}} \leq \frac{1}{\sqrt{1-\beta_2}}    
\end{equation}

Substituting the above bound on $T_1$ in
\eqref{eq: rmsprop_full}, using $\hat \g_t =\nabla \hat f(\w_t)$, we have the following:
\begin{align*}
    \mathbb{E}_t [\hat f(\w_{t+1})] &\leq \hat f(\w_t) - \eta_{t} \sum_{i=1}^{p} \frac{\left[\nabla \hat f(\w_t)\right]_{i}^{2}}{\sqrt{\beta_{2} \mathbf{v}_{t-1}^{i}}+\nu} +  \frac{\eta_tG\sqrt{1-\beta_2}}{\nu}\sum_{i =1}^{p}  \mathbb{E}_t \left[ \frac{(\left[\nabla \hat f(\w_t)\right]_{i} + \b_t^i)^{2}}{\sqrt{\beta_{2} \v_{t-1}^i}+\nu} \right] \\
    &  \qquad \qquad \qquad \qquad \qquad \qquad   \qquad + \frac{L \eta_{t}^{2}}{2 \nu} \sum_{i=1}^{d} \mathbb{E}_{t}\left[\frac{(\left[\nabla \hat f(\w_t)\right]_{i} + \b_t^i)^{2}}{\sqrt{\v_t^i}+\nu}\right] \\
    & \leq  \hat f(\w_t) - \eta_{t} \sum_{i=1}^{p} \frac{\left[\nabla \hat f(\w_t)\right]_{i}^{2}}{\sqrt{\beta_{2} \mathbf{v}_{t-1}^{i}}+\nu} +  \frac{\eta_tG\sqrt{1-\beta_2}}{\nu}\sum_{i =1}^{p}  \mathbb{E}_t \left[ \frac{(\left[\nabla \hat f(\w_t)\right]_{i} + \b_t^i)^{2}}{\sqrt{\beta_{2} \v_{t-1}^i}+\nu} \right] \\
    &  \qquad \qquad \qquad \qquad \qquad \qquad   \qquad + \frac{L \eta_{t}^{2}}{2 \nu} \sum_{i=1}^{d} \mathbb{E}_{t}\left[\frac{(\left[\nabla \hat f(\w_t)\right]_{i} + \b_t^i)^{2}}{\sqrt{\beta_2\v_{t-1}^i}+\nu}\right] \\
    & = \hat f(\w_t) -(\eta_{t}-\frac{\eta_{t} G \sqrt{1-\beta_{2}}}{\nu}-\frac{L \eta_{t}^{2}}{2 \nu}) \sum_{i=1}^{p} \frac{\left[\nabla \hat f(\w_t)\right]_{i}^{2}}{\sqrt{\beta_{2} \mathbf{v}_{t-1}^{i}}+\nu} + \left(\frac{\eta_{t} G \sqrt{1-\beta_{2}}}{\nu}+\frac{L \eta_{t}^{2}}{2 \nu}\right) \sum_{i=1}^{p} \frac{\sigma_{i}^{2}}{\sqrt{\beta_{2} \mathbf{v}_{t-1}^{i}}+\nu}.
\end{align*}

Given the parameter setting from the theorem, we see the following condition hold:
$\frac{L\eta_t}{\nu} \leq \frac{1}{2}$
and $\frac{G \sqrt{1-\beta_{2}}}{\nu} \leq \frac{1}{4}$. Let $\eta_t = \eta$, we obtain
\begin{align}
\nr \mathbb{E}_t \hat f(\w_{t+1})& \leq \hat f(\w_t) - \frac{\eta}{4} \sum_{i=1}^{p} \frac{\left[\nabla \hat f(\w_t)\right]_{i}^{2}}{\sqrt{\beta_{2} \mathbf{v}_{t-1}^{i}}+\nu} + \frac{\eta}{2} \sum_{i=1}^{p} \frac{\sigma_{i}^{2}}{\sqrt{\beta_{2} \mathbf{v}_{t-1}^{i}}+\nu} \\
& =  \hat f(\w_t) - \frac{\eta}{4} \frac{\left\|\nabla \hat f(\w_t)\right\|^{2}}{\sqrt{\beta_2\lambda} + \nu} + \frac{\eta}{2} \frac{p\sigma^2}{ \nu}
\end{align}

The second inequality follows from the fact that $0 \leq \v_{t-1}^i \leq \lambda$. Using the telescoping sum and rearranging the inequality, with $f^\star = \ell^\star$ in Assumption \ref{asmp: smoothness},  we obtain
\begin{align}
\mathbb{E} \| \nabla \hat f(\w_R) \|^2 = \frac{1}{T} \sum_{t=1}^T \mathbb{E}\|\nabla \hat f(\w_t)\|^2  \leq 4(\sqrt{\beta_2\lambda}+ \nu)\left(\frac{f(\w_1) -f^\star}{\eta T} + 2p\sigma^2\right).
\end{align}

Plugging in $T=O\left(\frac{ n \epsilon}{\sqrt{p \log (1 / \delta) }G}\right)$ and $\sigma^2 = O\left(\frac{G^{2} T \ln \left(\frac{1}{\delta}\right)}{n^{2} \epsilon^{2}}\right)$ achieves:
\begin{equation} 
    \mathbb{E} \| \nabla \hat f(\w_R) \|^2 \leq O\left(\frac{ G^2 \sqrt{p \log (1 / \delta)}}{n \epsilon}\right),
\end{equation}
where $\w_R$ is is uniformly
sampled from $\{\w_1, \w_2, ...,\w_T\}$. \qed

\subsection{Proof of Theorem \ref{thm:adam_fb}}
\label{sec:proof_opt_adam}
We restate the Theorem \ref{thm:adam_fb} here.

\theoadamfb*

Before we provide the proof of Theorem \ref{thm:adam_fb}, we first state the following lemma.

\begin{lemm} 
\label{lemm: acc_gauss}
Assume $\sigma$, $\epsilon$ and $\delta$ are set to satisfy the conditions in Theorem \ref{thm: acc_basic_fb} such that $\epsilon \leq \frac{\sigma}{13}$, $\delta \leq \frac{\sigma \exp(-\mu^2/2)}{13 \ln(26/\sigma)}$ and $n \geq \frac{2\ln(8/\delta)}{\epsilon^2}$, for the noisy gradients $\tilde \g_1,...,  \tilde \g_T$ in Algorithm \ref{algo: full batch}, we have $\forall t \in [T]$ and any $\mu > 0$:
\begin{equation}
    \mathbb{P}\{\|\tilde \g_t - \hat \g_t\| \geq \alpha \} \leq \xi,
\end{equation}
where $\alpha = \sqrt{p}\mu\sigma$ and $\xi = 2p\exp(-\mu^2/2)$.
\end{lemm}

\proof
Recall that $\tilde \g_t = \hat \g_t + \b_t$, where $\b_t$ is a noise vector drawn from Gaussian noise $\cN(0, \sigma^2\mathbb{I}_p)$. Using the tail bound of Gaussian random variable, we have
\begin{equation}
    \mathbb{P}\left\{ \|\tilde \g_t - \hat \g_t\| \geq  \sqrt{p}\sigma \mu \right\} \leq \mathbb{P} \left\{ \|\b_t\| \geq  \sqrt{p}\sigma \mu \right\} \leq \mathbb{P} \left\{ \|\b_t\|_\infty \geq  \sigma \mu \right\} \leq p \mathbb{P} \left\{ |\b_t^i| \geq \sigma \mu \right\} =  2p\exp(-\mu^2/2).
\end{equation}

The second inequality come from $\|\b_t\| \leq \sqrt{d}\|\b_t\|_\infty$.  we complete the proof.\qed

Now we present the proof of Theorem \ref{thm:adam_fb}.

\emph{Proof of Theorem \ref{thm:adam_fb}}:
The proof follows that of Theorem \ref{thm: gen_adam_fb} until \eqref{eq: adam_final} with the target function $f(\w)$ and gradient $\g_t$
changed to be empirical risk function $\hat f(\w)$ and empirical gradient $\hat \g_t$, so that we have
\begin{align}
    &\mathbb{E}\| \nabla \hat \hat f(\w_R)\|^2  \leq \frac{\sqrt{\lambda}+ \nu}{1-\beta_1} \lp \frac{2}{\eta T}(\hat f(\w_1) -  \hat f(\w_{T+1}))   + \frac{1}{(1-\beta_1)\nu}    \alpha^2 \rp
\end{align}
with probability at least $1-T\xi$.

Plugging in $\alpha = \sqrt{p}\mu\sigma$, $\xi = 2p\exp(-\mu^2/2)$
from Lemma \ref{lemm: acc_gauss}, $T = \frac{n\epsilon}{G \sqrt{p \ln(1/\delta)}}$, and $\sigma^2 = O\left(\frac{G^{2} T \ln \left(\frac{1}{\delta}\right)}{n^{2} \epsilon^{2}}\right)$ and setting $\mu = \sqrt{2\ln(2pT/\beta)}$ we have
\begin{align}
    \nr  \mathbb{E}\| \hat \nabla f(\w_R)\|^2  \leq  O\left( \frac{G^2\sqrt{p\ln (1/\delta)}\ln(n\sqrt{p}\epsilon/\beta)}{n\epsilon}\right),
\end{align}
with probability at least $1-\beta$.

\section{Uniform Convergence Lower Bounds}
\label{app:lower_bound}

We now show that there are simple loss function $\ell$ and distributions over $z$ for which the gradient deviation bound scales with $\sqrt{p}/\sqrt{n}$.

Let $\ell$ be loss function  such that for any $z\in \mathbb{R}^p$, $\ell(\w, \z) =  \frac{1}{2} \|\w - \z \|_2^2$. Suppose there are $n$ observations $\z_1, \ldots , \z_n$ drawn i.i.d. from a $p$-dimensional product distribution over $\{0, 1\}^p$. Let $\boldsymbol\mu = \E[\z]$. We assume that for each $j\in[p]$, $|\boldsymbol\mu_j| \in [1/3, 2/3]$. 

For any  $\w$ and any $\z_i$, the gradient $\g_i = \nabla \ell(\w, \z_i) = (\w - \z_i)$, and so $\E_{\z}\left[ \nabla \ell(\w, \z) \right] = \w - \boldsymbol\mu$. In other words, $\frac{1}{n}\sum_i \g_i(\w) - \E_{\z}\left[\nabla \ell(\w, \z)\right] = \boldsymbol\mu - \frac{1}{n}\sum_i \z_i$ for any $\w$.

\begin{theo}
Suppose that $X_1, \ldots , X_n$ are i.i.d. random variables. Let $\overline X = \frac{1}{n}\sum_{i=1}^n X_i$.

\begin{itemize}
\item \textbf{Multiplicative Chernoff.} Suppose that each $X_i\in [0, 1]$, then for any $\delta >0$, we have
\[
    \Pr[\E[X_1] - \overline{X} > \delta \E[X_1]] < \exp\left( - \E[X_1]n\delta^2 /3\right)
\]

    \item \textbf{Berry-Esseen.} Suppose that
\[
\sigma^2 = \E[(X_1 - \E[X_1])^2] \quad \mbox{and}\quad \rho = \E[|X_1 - \E[X_1]|^3] 
\]
Let $F_n$ be the cumulative distribution function of $\frac{(\overline X - \E[X_1]) \sqrt{n}}{\sigma}$ and $\Phi$ be the cumulative distribution function of the standard normal distribution. Then for all $x\in\mathbb{R}$,
\[
|F_n(x) - \Phi(x)| \leq \frac{\rho}{2\sigma^3 \sqrt{n}}
\]
\end{itemize}

\end{theo}



\begin{theo}
Suppose there are $n$ observations $\z_1, \ldots, \z_n$ drawn i.i.d.~from a product distribution over $\{0, 1\}^p$ such that the mean of each coordinate $\boldsymbol\mu_j\in [1/3, 2/3]$. Then with constant probability, for all $\w$, 
\[
 \left\| \frac{1}{n} \sum_{i=1}^n \nabla \ell(\w, \z_i) - \E_\z[\nabla \ell(\w, \z)] \right\|_2  \geq \Omega\left( \sqrt{\frac{p}{n}}\right)
\]
\end{theo}

\begin{proof}
 By applying Berry-Esseen theorem to each coordinate $j\in [p]$ with $X_i =  \left( \nabla_j \ell(\w, \z_i) - \E_\z[\nabla_j \ell(\w, \z)] \right)  = \boldsymbol\mu_j - \z_{ij}$, we have
\[
\Pr\left[\frac{\sqrt{n}}{\sigma} \left(\boldsymbol\mu_j - \frac{1}{n}\sum_{i=1}^n g_{ij} \right) > 1/2 \right] \geq \Phi(1/2) - \frac{\rho}{2\sigma^3\sqrt{n}}.
\]

There exists a constant $n_0$ such that for any $n\geq n_0$, $\Phi(1/2) - \frac{\rho}{2\sigma^3\sqrt{n}} \geq 1/10$.
Let $E_j$ denote the event that 
$\frac{\sqrt{n}}{\sigma} \left(\boldsymbol\mu_j - \frac{1}{n}\sum_{i=1}^n g_{ij} \right) > 1/2$, and $E$ be the event that $\sum_j \mathbf{1}[E_j] \geq p/20$. Then from the multiplicative Chernoff bound, let $X_j = \mathbf{1}[E_j]$, with $\delta = \frac{1}{2}$ and the fact that $\mathbb{E}[\mathbf{1}[E_j]] > \frac{1}{10}$,  we have
\[
\Pr\left[E \right] \geq 1 - \exp\left( - p/120\right) \geq 1/2, 
\]
where the last step holds for sufficiently large $p$. Then with probability at least $1/2$, we have
\[
\sqrt{\sum_j \left(\boldsymbol\mu_j - \frac{1}{n}\sum_{i=1}^n g_{ij}\right)^2}> \sqrt{\frac{p}{20} \left(\frac{\sigma}{2 \sqrt{n}}\right)^2} = \Omega(\sqrt{p/n})
\]
Then our theorem statement follows from the observation that $\g_i = \nabla \ell(\w, \z_i) = (\w - \z_i)$.

\end{proof}

\end{document}